\documentclass[11pt]{article}
\usepackage[T1]{fontenc}
\usepackage{graphicx, subcaption}
\usepackage{fullpage}
\usepackage{amssymb}
\usepackage{amsmath}
\usepackage{amsthm}
\usepackage{multirow}
\usepackage{enumerate}
\usepackage{graphicx}
\usepackage{mathrsfs}
\usepackage[utf8]{inputenc}
\usepackage{cite}
\usepackage{hyperref}
\usepackage{cleveref}
\usepackage{mathtools}
\usepackage{amsfonts}
\usepackage[T1]{fontenc}
\usepackage{mathpazo}
\usepackage{bm}
\usepackage{isomath}
\usepackage{verbatim}
\usepackage{changepage}
\usepackage[usenames,dvipsnames,svgnames,table]{xcolor}

\definecolor{blueviolet}{RGB}{60,50,200}
\definecolor{oliveg}{RGB}{40,200,30}
\hypersetup{colorlinks=true,       
    linkcolor=blueviolet,
    citecolor=oliveg,
}

\usepackage{tikz}
\usetikzlibrary{calc}

\usepackage{algorithmicx,algorithm}
\usepackage{algpseudocode}

\usepackage{color}

\usepackage{todonotes}

\theoremstyle{definition}

\theoremstyle{plain}
\newtheorem{lemma}{Lemma}[section]
\newtheorem{theorem}{Theorem}

\newtheorem{observation}{Observation}[section]

\makeatletter
\newcommand\footnoteref[1]{\protected@xdef\@thefnmark{\ref{#1}}\@footnotemark}
\makeatother

\newcommand{\MAB}{\textsc{Mab} }
\newcommand{\mab}{\textsc{mab}}
\newcommand{\PBA}{\texttt{PB-ALG} }
\newcommand{\pba}{\texttt{PB-ALG}}
\newcommand{\OBSALG}{\texttt{OBS-ALG} }
\newcommand{\obsalg}{\texttt{OBS-ALG}}
\newcommand{\GammaPBA}{\texttt{$\gamma$-NB-ALG} }
\newcommand{\Gammapba}{\texttt{$\gamma$-NB-ALG}}

\newcommand{\CRMPB}{\texttt{CRM-NB-ALG} }
\newcommand{\crmpb}{\texttt{CRM-NB-ALG}}
\newcommand{\R}{\mathbb{R}}
\newcommand{\N}{\mathbb{N}}
\newcommand{\UCB}{\texttt{UCB} }

\newcommand{\CUCB}{\texttt{C-UCB} }
\newcommand{\cucb}{\texttt{C-UCB}}
\newcommand{\CUCBTwo}{\texttt{C-UCB-2} }
\newcommand{\cucbtwo}{\texttt{C-UCB-2}}

	\ifdefined\DEBUG
	    \newcommand{\vineet}[1]{\textcolor{Red}{#1}}
	    \newcommand{\vishakha}[1]{\textcolor{OliveGreen}{#1}}
	    
	    \def\rem#1{{\marginpar{\raggedright\scriptsize #1}}}
	    \newcommand{\vin}[1]{\rem{\textcolor{Red}{$\bullet$ #1}}}
	    \newcommand{\vis}[1]{\rem{\textcolor{OliveGreen}{$\bullet$ #1}}}
	    \newcommand{\gau}[1]{\rem{\textcolor{NavyBlue}{$\bullet$ #1}}}
	\else
	    \newcommand{\vineet}[1]{#1}
	    \newcommand{\vishakha}[1]{#1}
	    
        \newcommand{\vin}[1]{}
	    \newcommand{\vis}[1]{}
	    \newcommand{\gau}[1]{}
	 
	\fi

\begin{document}
\title{Budgeted and Non-budgeted Causal Bandits}
\author{Vineet Nair\\
\footnotesize{Technion Israel Institute of Technology}\\
\normalsize{\tt vineet@cs.technion.ac.il}
\and {Vishakha Patil\footnote{These authors have made equal contribution and their names are alphabetically ordered.}}\\
\footnotesize{Indian Institute of Science}\\
\normalsize{\tt patilv@iisc.ac.in}
\and {Gaurav Sinha\footnotemark[\value{footnote}]}\\
\footnotesize{Adobe Research}\\
\normalsize{\tt gasinha@adobe.com}}
\maketitle
\date{}

\maketitle

\begin{abstract}

Learning good interventions in a causal graph can be modelled as a stochastic multi-armed bandit problem with side-information. First, we study this problem when interventions are more expensive than observations and a budget is specified. If there are no backdoor paths from an \emph{intervenable} node to the reward node then 
we propose an algorithm to minimize simple regret that optimally trades-off observations and interventions based on the cost of intervention. 
We also propose an algorithm that accounts for the cost of interventions, utilizes causal side-information, and minimizes the expected cumulative regret without exceeding the budget. Our cumulative-regret minimization algorithm performs better than standard algorithms that do not take side-information into account. Finally, we study the problem of learning best interventions without budget constraint in general graphs and give an algorithm that achieves constant expected cumulative regret in terms of the instance parameters when the parent distribution of the reward variable for each intervention is known. Our results are experimentally validated and compared to the best-known bounds in the current literature.

\end{abstract}
\section{Introduction}
\label{sec:Introduction}
Causal Bayesian Networks (CBN) \cite{PEARL2009} have become the popular choice to model causal relationships in many real-world systems such as online advertising, gene interaction networks, brain functional connectivity, etc. The underlying directed acyclic graph (DAG) of a CBN is called its \emph{causal graph}. The nodes of this graph are labeled by random variables\footnote{The joint distribution of these random variables factorizes over the graph.} representing the underlying system, and edges between these variables capture direct causal relationships.
Once the causal graph is known, any external manipulations on the system that forcibly fixes some target variables can be modeled via an operation called \emph{intervention}. An intervention simulates the effect of such a manipulation of the target variables on other system variables by disconnecting the target variables from their parents\footnote{A process known as \emph{causal surgery}.} and setting them to the desired value.   

Two key questions in causal learning are: 1) learning the causal graph itself, and 2) finding the intervention that optimizes some variable of interest (often called \emph{reward} variable) assuming that the causal graph is known. In this work, we focus on the second question by modeling the causal learning problem as an extension of the Stochastic Multi-armed Bandit Problem (\mab) \cite{ROBBINS1952}. The \MAB problem is a popular model used to capture decision-making in uncertain environments where a decision-maker is faced with $k$ choices (called \emph{arms}) and at each time step the decision-maker has to choose one out of the $k$ arms (\emph{pull an arm}). The arm that is pulled gives a reward drawn from an underlying distribution which is unknown to the decision-maker beforehand.

We study the \MAB problem with dependencies between the arms modelled via a causal graph. This model, called \emph{causal bandits}, 
was studied in the recent works of \cite{BareinboimFP15,LattimoreLR16,SenSDS17,SenSKDS17,LeeB18,YabeHSIKFK18,LeeB19,LU2020}, where the interventions are modelled
as the arms of the bandit and the influence of the arms on the reward is assumed to conform to a known causal graph. In addition to the possible interventions allowed, the set of arms also contains the empty intervention called the observational arm, where the algorithm does not perform any intervention on the causal graph. The goal of a causal bandit algorithm is to learn the intervention that maximizes the reward.
\begin{figure}[ht!]
    \centering
    \includegraphics[scale=0.44]{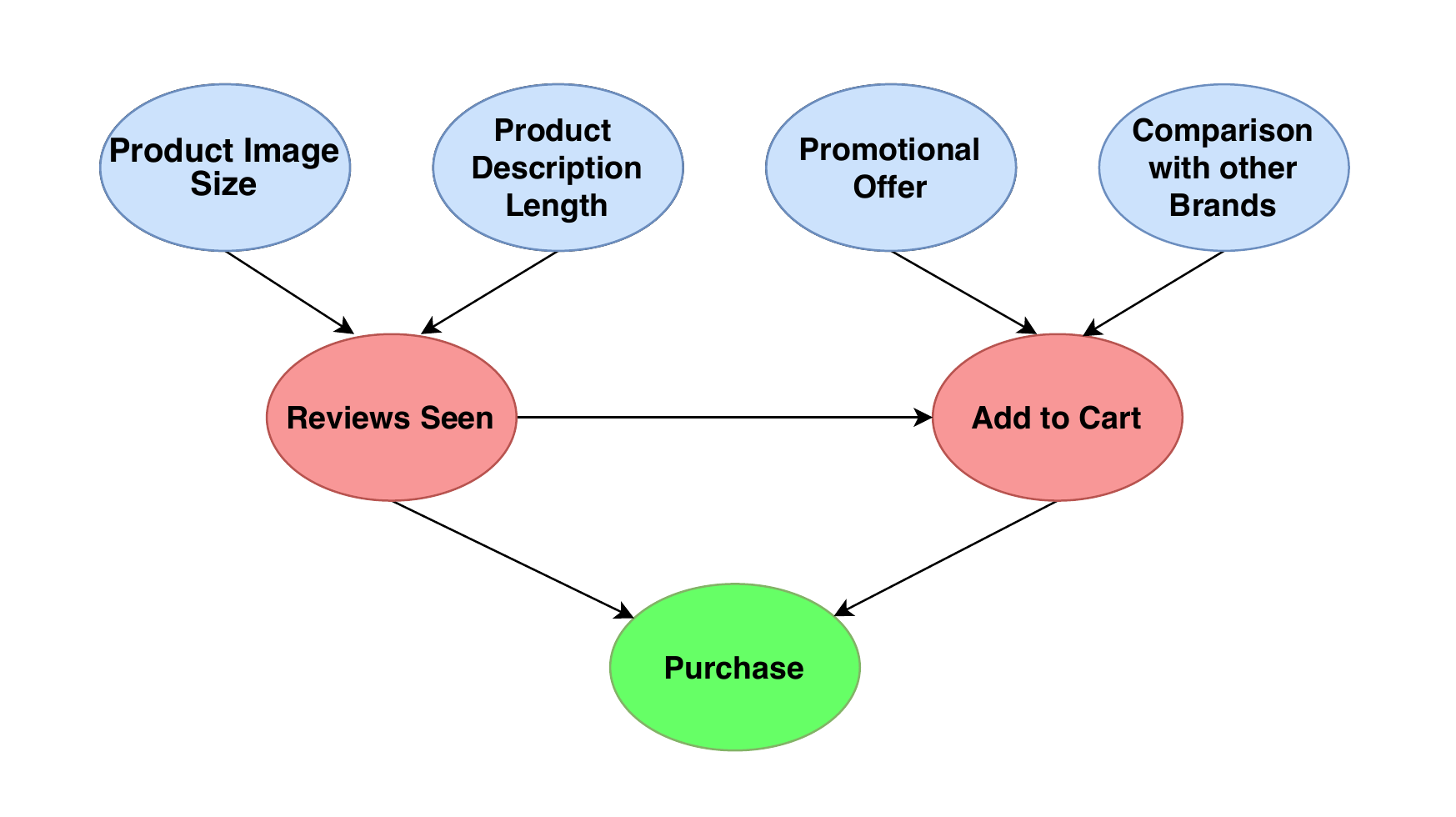}
    \caption{Causal Graph - Product Marketing}
    \label{fig:exampleGraph}
\end{figure}

We explain the causal bandit problem with a motivating example from the marketing domain for which a simple causal graph is shown in Figure \ref{fig:exampleGraph}. An e-commerce company sells a product online and makes a profit whenever a customer purchases their product. This corresponds to the green (reward) node labeled \emph{Purchase} in the causal graph. On every new customer visit, the product webpage is rendered using some values of the blue (intervenable) nodes, chosen from an underlying distribution.
For example, a customer might see a large image, small description, no promotions, and comparison with a competing brand. The red nodes capture actions taken by the customer before they make any purchase decision. For example, based on the rendered product page, a customer might want to get more information from the reviews before deciding to add the product to their shopping cart. Note that, while the blue nodes are actionable and can be manipulated to increase the chances of purchase, the red nodes are not directly manipulable and can only be passively observed for any given values of the blue nodes. For example, the company might take an action by always offering a promotion, seeing which customers might decide to skip going through the reviews and directly add the product to their shopping cart.
The objective here is to learn the intervention (on blue nodes) that maximizes the chances of the product being purchased.

However, in many situations, interventions are costly \cite{KocaogluDV17,KocaogluSB17,LindgrenKDV18,AddankiKMM20}. Consider the marketing example above where observational data from this graph can be collected via independent customer visits whereas to get an interventional sample, one needs to render a specific page configuration that would require additional expenditure. But recent works suggest that in many scenarios the effect of interventions can be efficiently estimated using observational samples  \cite{TianP02,BhattacharyyaGKMV20,PEARL2009}. Hence, in the causal bandit framework, for a fixed budget, there is a trade-off between the more economical observational arm and the high-cost interventional arm. This is because the observational arm, though less rewarding, aids in the exploration of the possibly high rewarding interventional arms. This motivates the study of observation/intervention trade-off in the budgeted bandit setting.

\subsection{Our Contributions}\label{subsec: contributions}
We study the problem of finding the best intervention in a causal graph in two settings: with and without budget constraints. Further, we study these problems with two objectives that are common in the \MAB literature: simple regret minimization and cumulative regret minimization. 

\noindent \textbf{Budgeted Setting}: In Sections \ref{sec: simple regret} and \ref{section: cum regret for parallel bandits} we consider a class of causal graphs that we call \emph{no-backdoor graphs} (see Section \ref{sec: simple regret} for the definition).
A special instance of the no-backdoor graph class is the parallel graph model defined in \cite{LattimoreLR16}: $\mathcal{G}$ consists of $M+1$ nodes, $\mathcal{X} = \{Y,X_1, \ldots , X_M\}$, and the only edges in $\mathcal{G}$ are from each $X_i$ to $Y$. For this, \cite{LattimoreLR16} propose an algorithm called the \emph{parallel bandit} algorithm (\PBA henceforth). We observe that \PBA in fact works for the more general class of no-backdoor graphs. 

We study the causal bandit problem for no-backdoor graphs in the budgeted bandit setting \cite{TRAN2012}, where a budget $B$ is specified and the ratio of the cost of the intervention to the cost of observation is $\gamma \geq 1$. \vishakha{The goal of an algorithm is to find the best intervention such that the total cost of arm pulls does not exceed $B$.}
In Section \ref{sec: simple regret}, we first study this problem with the goal of minimizing simple regret. 
Note that \PBA does not take into account the cost of interventions and is only optimal in the non-budgeted setting.
We show that when $\gamma$ is higher than a threshold (unknown to the algorithm), the simple algorithm \OBSALG that plays the observational arm every time achieves better simple regret in terms of $B$ than \pba. 
Next, we propose \GammaPBA (Algorithm \ref{algorithm: gamma parallel bandits}) which determines this unknown threshold online and successfully manages to trade-off interventions with observations for a specified budget. 


In Section \ref{section: cum regret for parallel bandits}, we study the cumulative regret minimization (CRM) problem in the above setting and give the \CRMPB algorithm. \CRMPB is based on the $\mathtt{Fractional-KUBE}$ algorithm ($\mathtt{F-KUBE}$ henceforth) given in \cite{TRAN2012} for budgeted bandits with no side-information. \CRMPB achieves constant regret if the observational arm is the optimal arm and otherwise achieves logarithmic regret which is better than that of $\mathtt{F-KUBE}$ in terms of instance-specific constants. 

\noindent \textbf{Non-Budgeted Setting}: In Section \ref{section: cumulative regret for general graphs}, we study the problem of minimizing the cumulative regret for general causal graphs in the non-budgeted setting. We assume that the distribution of parents of the reward variable for each intervention is known to the algorithm. This assumption though limiting in the practical setting is also made in the recent work of \cite{LU2020} (which studies the same problem) as well as in the work of \cite{LattimoreLR16}. 
\cite{LU2020} proposed an algorithm called \CUCB which has a worst-case regret guarantee of $O(\sqrt{k^nT)}$ where $k$ is the number of distinct values that each of the $n$ parents of the reward variable can take. For the same problem, we propose \CUCBTwo (Algorithm \ref{algorithm: modfified causal UCB}) and show it has constant expected cumulative regret in terms of instance parameters which is a significant improvement.
\section{Model and Notations}\label{sec: model and notation}
A CBN is a directed acyclic graph $\mathcal{G}$ whose nodes are labelled by random variables $\mathcal{X} = \{X_1, \ldots , X_n\}$, and a joint probability distribution $\mathbb{P}$ over $\mathcal{X}$ that factorizes over $\mathcal{G}$. For each $i\in[n]$, the 
range of $X_i$ is a finite subset of $\R$
. A node $X_j$ is called a parent of node $X_i$ if there is an edge from $X_j$ to $X_i$ in $\mathcal{G}$, i.e., changes in $X_j$ directly affect $X_i$. The set of parents of $X_i$ is denoted as $Pa(X_i)$. An intervention of size $m$ corresponds to $\mathbf{X} \subset \mathcal{X}$ such that $|\mathbf{X}| = m$, where the variables in $\mathbf{X}$ are 
set to $\mathbf{x} = (x_1, \ldots, x_m)$, and this intervention is denoted as $do(\mathbf{X} = \mathbf{x})$. For each $X_i \in \mathbf{X}$, the intervention also removes all the edges from $Pa(X_i)$ to $X_i$, and the resulting graph defines a probability distribution $\mathbb{P}(\mathbf{X}^c| do(\mathbf{X}=\mathbf{x}))$ over $\mathbf{X}^{c} = \mathcal{X}\setminus \mathbf{X}$. The empty intervention, also called \emph{observation}, corresponding to $\mathbf{X}=\phi$ is denoted as $do()$. A causal bandit algorithm is given as input a causal graph $\mathcal{G}$, the set of allowed interventions $\mathcal{A}$ (which corresponds to the set of arms), and a designated reward variable $Y\in \mathcal{X}$ where $Y\in \{0,1\}$. The distribution $\mathbb{P}$ is unknown to the algorithm. 

An algorithm for this problem is a sequential decision-making process that at each time $t$ performs an intervention $a_t \in \mathcal{A}$ and observes reward $Y_t \in \{0,1\}$. 
For each intervention $a \in \mathcal{A}$, where $a = do(\mathbf{X} = \mathbf{x})$, 
the expected reward of $a$ is denoted $\mu_{a} = E[Y \mid do(\mathbf{X} = \mathbf{x})]$. 
We study a budgeted as well as a non-budgeted variant of this problem. 
\subsection{Budgeted Causal Bandits}
In the budgeted variant of our problem, we associate a cost $\gamma > 1$ with each arm pull of a non-empty intervention $a\in \mathcal{A}\setminus \{do()\}$, whereas the cost of pulling the observation arm $do()$ is one. Hence, $\gamma$ is the ratio of the cost of a non-empty intervention to the cost of the observational arm.
The algorithm, in addition to $\mathcal{G}$ and $\mathcal{A}$, is specified a budget $B \in \mathbb{R}_+$. This model is similar to the budget limited bandit model considered in \cite{LongCCRJ10,TRAN2012} where each arm has an associated cost per pull. \vineet{Different cost models such as the linear cost model \cite{LindgrenKDV18} and identity cost model \cite{AddankiKMM20} have been studied in the causal discovery literature. The identity cost model is equivalent to considering a uniform cost $\gamma$ across all interventions. In this work, we are interested in the trade-off between observations and interventions which is more perceptible in the identity cost model. Moreover, the algorithm presented for the budgeted setting in this paper can be extended with a bit of effort to different cost models.}
We study this problem from the perspective of two objectives: minimization of simple regret and cumulative regret, both being well-studied in the bandit community. 

\textbf{Simple regret}: Let $\mathtt{ALG}$ be an algorithm for the above problem that outputs arm $a_B$ when the budget given is $B$. 
Then the simple regret of $\mathtt{ALG}$ with budget $B$, denoted $r(B)$, is 
\begin{equation}
    \label{eq:simple regret}
    r_{\mathtt{ALG}}(B) = \text{max}_{a\in \mathcal{A}}\mu_{a} - \mu_{a_B}
\end{equation}
An algorithm whose objective is to minimize the simple regret is a pure-exploration algorithm and its goal is to identify the best arm without having to restrict the number of times a sub-optimal arm may be played using the budget $B$. In many applications, we may require that a sub-optimal arm should not be pulled too many times right from the start. This motivates the definition of cumulative regret.

\textbf{Cumulative regret}: Let $G_\mathtt{ALG}(B)$ be the expected reward accumulated by algorithm $\mathtt{ALG}$ with budget $B$, and let $G_B = \max_{\mathtt{ALG}} G_\mathtt{ALG}(B)$. Then, the cumulative regret of an algorithm $\mathtt{ALG}$ with budget $B$, denoted $R_{\mathtt{ALG}}(B)$, is 
\begin{equation}
    \label{eq:cumulative regret}
   R_{\mathtt{ALG}}(B) = G_B - G_\mathtt{ALG}(B)
\end{equation}

An algorithm for cumulative regret minimization has to carefully trade-off between exploration vs. exploitation. Hence, an algorithm with good simple regret guarantees may not have good cumulative regret guarantees and vice-versa. 

\subsection{Non-budgeted Causal Bandits}\label{subsec: non-budegeted bandits}
In the non-budgeted variant of the problem, the cost associated with every intervention is the same, i.e., we can assume $\gamma = 1$ for all $a\in\mathcal{A}$. 
In Section \ref{section: cumulative regret for general graphs}, we study this problem with the objective of minimizing the expected cumulative regret when the time horizon $T$ is unknown but finite. The regret notion is defined as in Equation \ref{eq:cumulative regret} with $B=T$. Observe that $G_T = T\cdot \max_{a\in \mathcal{A}} \mu_a$. The cumulative regret of an algorithm $\mathtt{ALG}$ after $T$ rounds, denoted $R_{\mathtt{ALG}}(T)$, is then defined as
\begin{equation}
    \label{eq:cumulative regret without cost}
    R_{\mathtt{ALG}}(T) = T\cdot\max_{a\in \mathcal{A}} \mu_a - \sum_{t\in [T]} \mu_{a_t}~.
\end{equation}
The goal of any algorithm for such a setting is to minimize the expected cumulative regret $E[R_{\mathtt{ALG}}(T)]$ where the expectation is taken over the randomness in the rewards as well as in the algorithm.

\section{Budgeted \mab: No-backdoor Graphs}\label{sec: simple regret}

In this section and Section \ref{section: cum regret for parallel bandits}, we assume that the interventions are of size $1$. Formally, let $\{X_1, \ldots,$  $X_M\}$ $\subseteq \mathcal{X}$ be the set of intervenable nodes such that $X_i \in \{0,1\}$ for all $i \in [M]$. An intervention in this setting is defined as explicitly setting the value of a single node $X_i$ as either $0$ or $1$. When not intervened upon, $X_i \sim \text{Bernoulli}(p_i)$. Hence, we have $2M+1$ interventions in total: $2M$ interventions correspond to setting each of the $M$ variables $X_i$ to either $0$ or $1$, denoted $do(X_i = 0)$ and $do(X_i = 1)$ respectively, and the last intervention corresponds to the empty intervention, $do()$. Moreover, we assume that there are no backdoor paths from $X_i$ to the reward variable. This implies $E[Y \mid do(X_i=x)] = E[Y\mid X_i=x] = \mu_{i,x}$ (see Section 3.3.1 \cite{Pearl00}). We call a causal graph $\mathcal{G}$ satisfying this property as a \emph{no-backdoor graph} ( $\mathtt{NB}$-graph).

For ease of notation, we denote the intervention $do(X_i = x)$ by $a_{i,x}$ where $i\in[M]$ and $x\in \{0,1\}$, and the empty intervention as $a_0$. The set of interventions is then $\mathcal{A} = \{a_{i,x} \mid i\in [M], x\in \{0,1\}\} \uplus \{a_0\}$. The expected reward for the intervention $a_{i,x}$ and $a_0$ are $\mu_{i,x} = E[Y\mid X_i = x]$ and $\mu_0 = E[Y]$ respectively. Throughout Sections \ref{sec: simple regret} and \ref{section: cum regret for parallel bandits}, $i\in [M]$ and $x\in \{0,1\}$. Also, we use $a$ to denote an intervention in $\mathcal{A}$ when we do not differentiate between $a_{i,x}$ and $a_0$.
We study the budgeted causal bandit problem for no-backdoor graphs. As stated in Section \ref{sec: model and notation}, an algorithm for this problem is given as input the graph $\mathcal{G}$, the set of intervenable nodes $\{X_1, \ldots, X_M\}$, a budget $B$, and $\gamma$ which is the cost for pulling an arm $a_{i,x}$. The algorithm does not know $p_i$ for any $i$. Note that if $\gamma \geq B$ then trivially the algorithm can only make observations.

As stated above, $X_i \sim \text{Bernoulli}(p_i)$. 
Let $p_{i,1} = \mathbb{P}(X_i = 1) = p_i$ and $p_{i,0} = 1- p_{i,1}$. We assume $p = \min_{(i,x)} \{p_{i,x}\} >0$, \vineet{which is reasonable in situations where the best arm is observable, and if the best arm is not observable then observational samples are not useful}. Also let $\mathbf{p} = (p_1 \ldots p_M)$. \pba, the algorithm by \cite{LattimoreLR16} minimizes the expected simple regret for the parallel graph model in the non-budgeted setting. We observe that \PBA works for any no-backdoor graph model. \vineet{In particular, \PBA when applied to the budget setting plays the observational arm $a_0$ for the first $\frac{B}{(1+\gamma)}$ rounds and in the remaining $\frac{B}{(1+\gamma)}$ rounds plays the interventional arms that were observed fewer times during the arm pulls of $a_0$ (because the probability of observing them when $a_0$ is pulled is low).} 
The simple regret guarantee of \PBA depends on the quantity $m(\mathbf{p})$ which captures the number of $X_i$'s such that $\text{min}(p_i,1-p_i) \ll 1/2$, \vineet{(i.e. the number of arms that would be observed fewer number of times when the arm $a_0$ is pulled in the initial rounds)}. 
Formally, for $\tau \in [2,M]$ let $I_{\tau} = \big\{i \mid \min_x \{p_{i,x}\} <\frac{1}{\tau}\big\}$. Then $m(\mathbf{p}) = \min\{\tau \mid |I_\tau| \leq \tau\}$ and the simple regret of the \PBA algorithm is $O\Big(\sqrt{\frac{\gamma m(\mathbf{p})}{B}\log (\frac{MB}{\gamma m})}\Big)$. 

In Section \ref{subsec: obs algorithm}, we show that for $\gamma = \Omega(\frac{1}{p\cdot m(\mathbf{p})})$ the simple algorithm that plays the observation arm for $B$ rounds achieves better expected simple regret than \pba. Since $p_i$ for all $i$ is unknown, the threshold $\frac{1}{p\cdot m(\mathbf{p})}$ is a priori unknown to an algorithm. Hence, in Section \ref{subsec: obs intervention tradeoff} we propose an algorithm that estimates this threshold online, and trades-off between interventions and observations dependent on $\gamma$ and the threshold to minimize the expected simple regret. 

\subsection{The observational algorithm}\label{subsec: obs algorithm}
Here, we analyze the simple-regret of the observational algorithm (\obsalg) which plays the arm $a_0$ for all the rounds, and at the end of $B$ rounds outputs the arm $a\in \mathcal{A}$ with the highest empirical mean estimate. The empirical estimate of $\mu_{i,x}$ is computed as the average of the rewards accrued in those rounds where $X_i$ was sampled as $x$. Theorem \ref{theorem: simple regret observational algorithm} shows the dependence of the expected simple regret of \OBSALG on $p$.  
\begin{theorem}\label{theorem: simple regret observational algorithm}
The expected simple regret of \OBSALG with budget $B$ is $O\left(\sqrt{\frac{1}{pB}\log (pMB)}\right)$.
\end{theorem}
The proof of Theorem \ref{theorem: simple regret observational algorithm} is in Section \ref{secappendix: proof of obs simple regret}. Theorem \ref{theorem: simple regret observational algorithm} is proved by crucially leveraging the fact that the arm $a_0$ aides in the exploration of all the other $2M$ arms, which is the side-information available in $\mathtt{NB}$-graphs. Observe that the guarantee of observational algorithm is better than that of \PBA in \cite{LattimoreLR16} if $\gamma = \Omega(\frac{1}{p\cdot m(\mathbf{p})})$.

\subsection{Observation-Intervention Trade-off}\label{subsec: obs intervention tradeoff}
\Gammapba~ (Algorithm \ref{algorithm: gamma parallel bandits}) trades-off between observations and interventions depending on the value of $\gamma$ to minimize the expected simple regret. The idea behind \Gammapba~ is that if $\gamma$ is larger than the threshold $\frac{1}{p\cdot m(\mathbf{p})}$ then performing only observations gives a better regret (as stated at the end of Section \ref{subsec: obs algorithm}), whereas if $\gamma$ is less than this threshold then the algorithm follows the strategy of \PBA by playing the interventions in set $A$ (see step 11 of \GammaPBA) for an equal number of times in the remaining rounds. 
At steps 13-14, the empirical estimates of only the arms in $A$ are updated. Since $\mathbf{p}$ and $p$ are not known a priory, the algorithm has to estimate the threshold online as done in Step 6 of \Gammapba. Note that at step 6, $\widehat{\mathbf{p}} = (\widehat{p}_{1} \ldots \widehat{p}_M)$, where $\widehat{p}_i = \widehat{p}_{i,1}$, and $m(\widehat{\mathbf{p}})$ is defined similar to $m(\mathbf{p})$.  

\begin{algorithm}[ht!]
\caption{\Gammapba} \label{algorithm: gamma parallel bandits}
\begin{algorithmic}
\State INPUT: $\mathcal{G}$, $B$, and $\gamma$.
\end{algorithmic}
\begin{algorithmic}[1]
\State Play arm $a_0$ for the first $B/2$ rounds.
\State For each $a\in \mathcal{A}$, compute \\
\hspace{0.5cm}$\widehat{\mu}_{i,x} = \frac{\sum_{t=1}^{B/2}Y_t\cdot \mathbb{1}\{X_i = x\}}{\sum_{t=1}^{B/2}\mathbb{1}\{X_i = x\}}$,\hspace{0.3cm}$\widehat{\mu}_0 = \frac{2\sum_{t=1}^{B/2} Y_t}{B}$
\State For each $(i,x)$, compute\\
\hspace{0.5cm} $\widehat{p}_{i,x} = \frac{2}{B}\sum_{t=1}^{B/2}\mathbb{1}\{X_i = x\}$, and $\widehat{p} = \min_{i,x}\widehat{p}_{i,x}$
\If{$\widehat{p} \cdot m(\widehat{\mathbf{p}}) \geq \frac{1}{\gamma}$}
    \State Play arm $a_0$ for the remaining $B/2$ rounds.
    \State For each $a\in \mathcal{A}$, compute \\
    \hspace{0.7cm}$\widehat{\mu}_{i,x} = \frac{\sum_{t=1}^{B}Y_t\cdot \mathbb{1}\{X_i = x\}}{\sum_{t=1}^{B}\mathbb{1}\{X_i = x\}}$,\hspace{0.3cm}$\widehat{\mu}_0 = \frac{\sum_{t=1}^T Y_t}{B}$
\Else
    \State Compute $A = \{a_{i,x} \mid \widehat{p}_{i,x} < \frac{1}{m(\widehat{\mathbf{p}})} \}$.
    \State Play each arm $a_{i,x} \in A$, for $\frac{B}{2\gamma |A|}$ rounds. 
    \State For each $a_{i,x} \in A$, set \\ \hspace{1.5cm}$\widehat{\mu}_{i,x} = \frac{2\gamma |A|}{B}\sum_{t=\frac{B}{2}+1}^{\frac{B}{2}+ \frac{B}{2\gamma}} Y_t\cdot \mathbb{1}\{a_t = a_{i,x}\}$.
\EndIf
\State Output $\arg\max_{a\in \mathcal{A}} \widehat{\mu}_a$.
\end{algorithmic}
\end{algorithm}
In Theorem \ref{theorem: simple regret of gammapba}, we bound the expected simple regret of \GammaPBA which depends upon $\gamma$ and the value of the threshold.
\begin{theorem} \label{theorem: simple regret of gammapba}
If $\gamma \geq \frac{1}{p\cdot m(\mathbf{p})}$ then the expected simple regret of \GammaPBA is $O\left(\sqrt{\frac{1}{pB}\log (pMB)}\right)$, and if $\gamma \leq \frac{1}{p\cdot m(\mathbf{p})}$ then it is $O\left(\sqrt{\frac{\gamma\cdot m(\mathbf{p})}{B}\log \frac{MB}{\gamma \cdot m(\mathbf{p})}}\right)$.
\end{theorem}
The proof of Theorem \ref{theorem: simple regret of gammapba} is in Section \ref{secappendix: proof of gamma simple regret}. Observe that the expected simple regret of \GammaPBA is equal to that of \PBA if $\gamma \leq \frac{1}{p\cdot m(\mathbf{p})}$, and is equal to the that of \OBSALG if $\gamma > \frac{1}{p\cdot m(\mathbf{p})}$. For $\gamma = O(\frac{1}{p\cdot m(\mathbf{p})})$ the optimality of the regret up to log factors follows from Theorem 2 in \cite{LattimoreLR16} where they show a  $\Omega\left(\sqrt{\frac{\gamma m}{B}}\right)$ lower bound on the expected simple regret.\footnote{The lower bound is shown in non-budgeted setting, which translates to $\Omega\left(\sqrt{\frac{\gamma m}{B}}\right)$ lower bound in our setting if $\gamma = O(\frac{1}{p\cdot m(\mathbf{p})})$.} \vineet{The experiment 2 in Section \ref{sec: experiments} shows that the performance of \GammaPBA matches or is better than the performance of \PBA for all values of $\gamma$, which validates our theoretical claim.}

\section{Cumulative-Regret in No-backdoor Graphs with Budget}\label{section: cum regret for parallel bandits}
In this section, we give the algorithm \CRMPB (Algorithm \ref{algorithm: cum regret for parallel bandits}) that minimizes the cumulative regret for the model in Section \ref{sec: simple regret}. 
\CRMPB is based on $\mathtt{Fractional-KUBE}$ ($\mathtt{F-KUBE}$, \cite{TRAN2012}), which is a budget-limited version of Upper Confidence Bound ($\mathtt{UCB}$, \cite{AUER2002}) but without side-information. 
In our setting, since arm $a_0$ aides in the exploration of all other $2M$ arms and has a unit cost, \CRMPB unlike $\mathtt{F-KUBE}$ ensures that arm $a_0$ is pulled sufficiently many times. Also importantly, in \CRMPB the estimate for an arm $a_{i,x}$ is made using the effective number of pulls of the arm $a_{i,x}$, which is equal to the number of pulls of the arm $a_{i,x}$ plus the number of pulls of arm $a_0$ where $X_i$ was sampled as $x$. 
Formally, let $N^{i,x}_{t}$ and $N^0_{t}$ denote the number of pulls of arm $a_{i,x}$ and $a_0$ respectively after $t$ rounds, and let $a_t$ denote the arm pulled at round $t$. The effective number of arm pulls of $a_{i,x}$ after $t$ rounds is equal to $E^{i,x}_t = N^{i,x}_t + \sum_{s=1}^t\mathbb{1}\{a_s = a_0 \text{~and~} X_i = x\}$. 

At the end of $t$ rounds \CRMPB computes $\widehat{\mu}_{i,x}(t)$ and $\widehat{\mu}_0(t)$ which are empirical estimates of $\mu_{i,x}$ and $\mu_0$ respectively, as follows: 
$$\widehat{\mu}_{i,x}(t) = \frac{\sum_{s=1}^{t} Y_s\cdot \mathbb{1}\{(a_s = a_{i,x}) \text{\,or\,} (a_s = a_0 \text{~and~} X_i = x)\}}{E^{i,x}_t} $$ 
$$\widehat{\mu}_{0}(t) =  \frac{1}{N^{0}_t}\sum_{s=1}^t Y_s\cdot \mathbb{1}\{a_s = a_0\} \big)$$
Based on this estimates the \CRMPB computes the weighted UCB estimate $\overline{\mu}_{i,x}(t)$ and $\overline{\mu}_{0}(t)$ for the arms $a_{i,x}$ and $a_0$ as follows:
$$\overline{\mu}_{i,x}(t) = \frac{1}{\gamma}\bigg(\widehat{\mu}_{i,x}(t) + \sqrt{\frac{8\log t}{E^{i,x}_{t}}}\bigg) $$
$$\overline{\mu}_{0}(t) = \widehat{\mu}_{0}(t) + \sqrt{\frac{8\log t}{N^0_{t}}} $$
In each round \CRMPB first ensures arm $a_0$ is pulled at least $\beta^2 \log T$ times (steps 4-5), where $\beta$ is set as in steps 11-14  and otherwise pulls the arm with the highest weighted UCB estimate (steps 6-8). 
\vineet{Ensuring the arm $a_0$ is pulled at least $\beta^2 \log T$ times at the end of $T$ rounds delicately balances the exploration-exploit trade-off: the causal side-information by pulling the arm $a_0$ ensuring free exploration of the other $2M$ interventions and the loss experienced in pulling the arm $a_0$ (if $a_0$ is the sub-optimal arm).}  
The reason for setting $\beta$ as in steps 11-14 is explained after Theorem \ref{theorem: cumulative regret of gamma paralllel bandit}, which bounds the expected cumulative regret of \crmpb. Observe that \CRMPB halts once it has exhausted its entire budget $B$. Crucially though, the decisions of \CRMPB  do not depend on the budget, i.e., it is budget oblivious. But note that the decisions of the algorithm do take into account the cost of an intervention, i.e., the algorithm is not cost-oblivious. 

\begin{algorithm}[!ht]
\caption{\crmpb} \label{algorithm: cum regret for parallel bandits}

\begin{algorithmic}
\State INPUT: $\mathcal{G}$, Set of nodes $\{X_1,\ldots,X_M\}$, $B$, $\gamma$
\end{algorithmic}
\begin{algorithmic}[1]
\State Pull each arm once and set $t = 2M+2$
\State Update $B_{t-1} = B - 2\gamma M - 1$ and let $\beta = 1$.  
\While{$B_t \geq 1$}
\If{$N^{0}_{t-1} < \beta^2 \log t$ or $B_{t}< \gamma$} 
    \State Pull $a_t = a_0$
\Else 
    \State Pull $a_t = \arg\max_{a\in \mathcal{A}} \overline{\mu}_{a}(t-1)$
\EndIf
\State Update $N^a_{t} = N^a_{t-1} + \mathbb{1}\{a_t = a\}$
\State Update $E^{a}_{t}$, $\widehat{\mu}_{a}(t)$ and $\overline{\mu}_{a}(t)$ for all $a \in \mathcal{A}$.  
\State Let $\widehat{\mu}^{*} = \text{max}_{i,x}\widehat{\mu}_{i,x}(t)$. 
\If{$\widehat{\mu}_0 (t) < \frac{\widehat{\mu}^*}{\gamma}$}
    \State Update $\beta = \min(\frac{2\sqrt{2}}{\widehat{\mu}^*/\gamma - \widehat{\mu}_0 (t)}, \sqrt{\log t})$
\EndIf
\State Update $B_{t+1} = \begin{cases}
            B_{t} - 1 & \text{if~} a_t = a_0\\
            B_{t} - \gamma & \text{if~} a_t \neq a_0\\
    \end{cases}$
\State Set $t=t+1$
\EndWhile
\end{algorithmic}
\end{algorithm}
Before stating Theorem \ref{theorem: cumulative regret of gamma paralllel bandit}, we introduce a few more notations which are used in the theorem. Let $v_{i,x} = \frac{\mu_{i,x}}{\gamma}$, and $v_0 = \mu_0$, and $a^* = \arg\max_{a\in \mathcal{A}}\{v_a\}$. Further, let $\Delta_{a} = \mu_{a^*} - \mu_a$ and $d_a = v_{a^*} -v_a$ for each $a \in \mathcal{A}$. Note that there could be $a\in \mathcal{A}$ such that $\Delta_a <0$. 

\begin{theorem}\label{theorem: cumulative regret of gamma paralllel bandit}
If $a^* = a_0$ then the expected cumulative regret of the algorithm is $O(1)$ and otherwise the expected cumulative regret of the algorithm is of order $\sum_{\Delta_{i,x}>0} \Delta_{i,x}\left(\text{max}\left(0, 1+ 8\ln B\left(\frac{1}{d_{i,x}^2} - \frac{p_{i,x}}{3d_0^2}\right) \right) + \frac{\pi^2}{3}\right) + \Delta_0 \left(\frac{50\ln B}{d_{0}^2} + 1 + \frac{\pi^2}{3} \right)$.
\end{theorem}
The optimal arm $a^*$ is equal to $a_0$ if the ratio of the expected reward of any intervention to expected reward of $a_0$ is at most $\gamma$. In particular, if $\frac{\text{max}_{i,x}\mu_{i,x}}{\mu_0} \leq \gamma$ then $a^* = a_0$ and in this case the expected cumulative regret of \CRMPB is bounded by a \emph{constant}. The proof of Theorem \ref{theorem: cumulative regret of gamma paralllel bandit} is given in Section \ref{secappendix: proof of gamma cumulative regret}. For the value of $\beta$ set as in steps 11-14, we show that if $a^* \neq a_0$ then $\frac{8}{9d_0^2} \leq E[\beta^2] \leq \frac{50}{d_0^2}$ (see Lemma \ref{lemma: bounding beta} in Section \ref{secappendix: proof of gamma cumulative regret}). This in particular ensures that if $a^*\neq a_0$ then the expected number of pulls of a sub-optimal arm $a_{i,x}$ is at most  $\text{max}\left(0, 1+ 8\ln B\left(\frac{1}{d_{i,x}^2} - \frac{p_{i,x}}{3d_0^2}\right)\right) + \frac{\pi^2}{3}$. Hence note that, if $\frac{1}{d_{i,x}^2} \geq \frac{p_{i,x}}{3d_0^2}$ then this sub-optimal arm is pulled at most a constant number of times. \vineet{In Section \ref{sec: experiments}, we show via simulation that \CRMPB performs much better than $\mathtt{F-KUBE}$ even for small values of $\gamma$. Note that $\mathtt{F-KUBE}$ does not take side information into account.}

\section{Cumulative Regret in General Graphs} \label{section: cumulative regret for general graphs}
In this section, we study the non-budgeted version of the causal bandit problem for general graphs (see Section \ref{subsec: non-budegeted bandits}) with the goal of minimizing the expected cumulative regret. This problem was studied in the recent work of \cite{LU2020} who gave a \UCB based algorithm, called \cucb, which has a worst-case regret bound of $\sqrt{k^nT}$ when each of the $n$ parent nodes of $Y$ (the reward variable) in the graph can take one of $k$ values. For the same problem, we propose an algorithm called \CUCBTwo, which
has constant regret in terms of instance-parameters. Additionally, \CUCB in \cite{LU2020} takes the time horizon $T$ as input, but our algorithm \CUCBTwo works for any \emph{unknown} (but finite) time horizon.

Let $Y_1, \ldots, Y_n$ be the parents of $Y$, hence we have $|Pa(Y)| = n$. Further, let $S \subset \R$ and $|S| = k$ be the set of values that a parents node of $Y$ can take. We denote the realization of $Y_i = y_i$, where $y_i\in S$ for each $i\in [n]$, as $Pa(Y) = \mathbf{y}$ where $\mathbf{y} = (y_1 \ldots y_n) \in S^n$. 
We assume the following: (a) the algorithm receives as input $\mathbb{P}(Pa(Y)= \mathbf{y}| do(a))$ for all $a$, and (b) the distributions $\mathbb{P}(Pa(Y)= \mathbf{y}| do(a))$ for all $a$ have the same non-zero support. \vineet{Assumption (a) is also made in \cite{LU2020, LattimoreLR16} whereas Assumption (b) is made in other existing literature on causal bandits (see \cite{SenSDS17}).} Let $c_{\mathbf{y}} = \text{min}_{a} (\mathbb{P}(Pa(Y) = \mathbf{y} \mid do(a))$. Observe that the expected reward $\mu_a$ of any intervention $a \in \mathcal{A}$ satisfies
$$ \mu_a = \sum_{\mathbf{y} \in S^n} E[Y\mid Pa(Y)= \mathbf{y}] \cdot \mathbb{P}(Pa(Y)= \mathbf{y} \mid do(a))~.$$
We denote $E[Y\mid Pa(Y)= \mathbf{y}]$ as $\mu_{\mathbf{y}}$. In \CUCBTwo (Algorithm \ref{algorithm: modfified causal UCB}), 
$\zeta_a = \sum_{c_{\mathbf{y}}>0} \frac{\mathbb{P}(Pa(Y)= \mathbf{y} \mid do(a))}{c_{\mathbf{y}}}$ for each $a$, and $N_{\mathbf{y},t}$ denotes the number of times $Pa(Y)$ have been sampled as $\mathbf{y}$ in $t$ rounds. Let $Pa_t(Y)$ denote the realization of $Pa(Y)$ at time $t$. Then, $N_{\mathbf{y},t} = \sum_{s=1}^t \mathbb{1}\{Pa_s(Y) = \mathbf{y}\}$ and
$\widehat{\mu}_{\mathbf{y}}(t)$ is the empirical estimate of $\mu_{\mathbf{y}}$ at the end of $t$ rounds defined as,
$$\widehat{\mu}_{\mathbf{y}}(t) = \frac{1}{N_{\mathbf{y},t}}\sum_{s=1}^t Y_s \cdot \mathbb{1}\{Pa_s(Y) = \mathbf{y}\} .$$
if $N_{\mathbf{y},t} \geq 1$ and otherwise $\widehat{\mu}_{\mathbf{y}}(t) = 0$. The algorithm also computes the empirical estimate $\widehat{\mu}_a(t)$ and the UCB estimate $\overline{\mu}_a(t)$ for all $a$ using $\widehat{\mu}_{\mathbf{y}}(t)$ at the end of every round as follows: 
$$\widehat{\mu}_a(t) = \sum_{\mathbf{y}\in S^n} \widehat{\mu}_{\mathbf{y}}(t)\cdot \mathbb{P}(Pa(Y)=\mathbf{y} | do(a)), \text{~~and}$$ $$\overline{\mu}_a(t) = \widehat{\mu}_a(t) + \sqrt{\frac{\log (k^n t^2/2)}{t}}\zeta_a~.$$
The quantity $\sqrt{\frac{\log (k^n t^2/2)}{t}}\zeta_a$ is called the upper confidence radius around the empirical estimate $\widehat{\mu}_a(t)$ at the end of $t$ rounds. We remark here that the difference between our algorithm \CUCBTwo and \CUCB by \cite{LU2020} is that \CUCB maintains a UCB estimate for each parent value tuple $\mathbf{y}$, whereas \CUCBTwo maintains a UCB estimate for each intervention $a\in \mathcal{A}$. 
\begin{algorithm}[!ht]
\caption{\CUCBTwo} \label{algorithm: modfified causal UCB}
\begin{algorithmic}
\State INPUT: $\mathcal{G}$, $\mathbb{P}(Pa(Y)=\mathbf{y} | do(a))$ for all $a\in \mathcal{A}$.
\end{algorithmic}
\begin{algorithmic}[1]
\State Play each intervention in round robin and for $t= |A|$ update $N_{\mathbf{y},t}, \widehat{\mu}_{\mathbf{y}}(t), \widehat{\mu}_a(t), \overline{\mu}_a(t)$.
\State Set $t=|A|+1$
\Loop
\State Play $a_t = \arg\max_{a\in A} \overline{\mu}_a$.
\State Update $N_{\mathbf{y},t}, \widehat{\mu}_{\mathbf{y}}(t), \widehat{\mu}_a(t), \overline{\mu}_a(t)$.
\State $t=t+1$
\EndLoop
\end{algorithmic}
\end{algorithm}
Theorem \ref{theorem: cumulative regret for general graphs} bounds the expected cumulative regret of \CUCBTwo. In Theorem \ref{theorem: cumulative regret for general graphs}, $\Delta_a = \max_{b\in \mathcal{A}} \mu_{b} - \mu_a$.
\begin{theorem}\label{theorem: cumulative regret for general graphs}
Let $\delta = \min_{\mathbf{y}\in S^n} \{c_{\mathbf{y}} > 0\}$, $L_1 = \min\Big\{t \in \N \mid t \geq \frac{2\log(k^nt^2)}{\delta^2} \Big\}$, $L_{2,a} = \min\Big\{t\in \N \mid t \geq \frac{4\log (k^{n} t^2/2)}{\Delta_a^2}\zeta_a^2 \Big\}$ for all $a\in \mathcal{A}$, and $L_a = \text{max}\{L_1,L_{2,a}\}$. Then the expected cumulative regret of \CUCBTwo after $T$ rounds is at most
$ \sum_{a\in \mathcal{A}} \Delta_a \Big(L_a + \frac{2\pi^2}{3}\Big).$
\end{theorem}
Observe that in Theorem \ref{theorem: cumulative regret for general graphs}, $L_a$ is a constant based on problem instance parameters, and hence Theorem \ref{theorem: cumulative regret for general graphs} proves that \CUCBTwo achieves instance dependent constant regret. Theorem \ref{theorem: cumulative regret for general graphs} is proved by showing that the expected number of pulls of a sub-optimal arm $a\in \mathcal{A}$ after time $L_a$ is at most $\frac{2\pi^2}{3}$ (proof in Section \ref{secappendix: proof of cum regret for general graphs}).
\vineet{In Section \ref{sec: experiments}, we show via simulations that the expected cumulative regret of \CUCBTwo is better than that of \cucb, and the experiment also validates that the regret of \CUCBTwo is a constant.}

\section{Algorithm Simulations}\label{sec: experiments}
\begin{figure}[ht!]
    \centering
    \minipage{0.45\textwidth}
    \includegraphics[width=\linewidth]{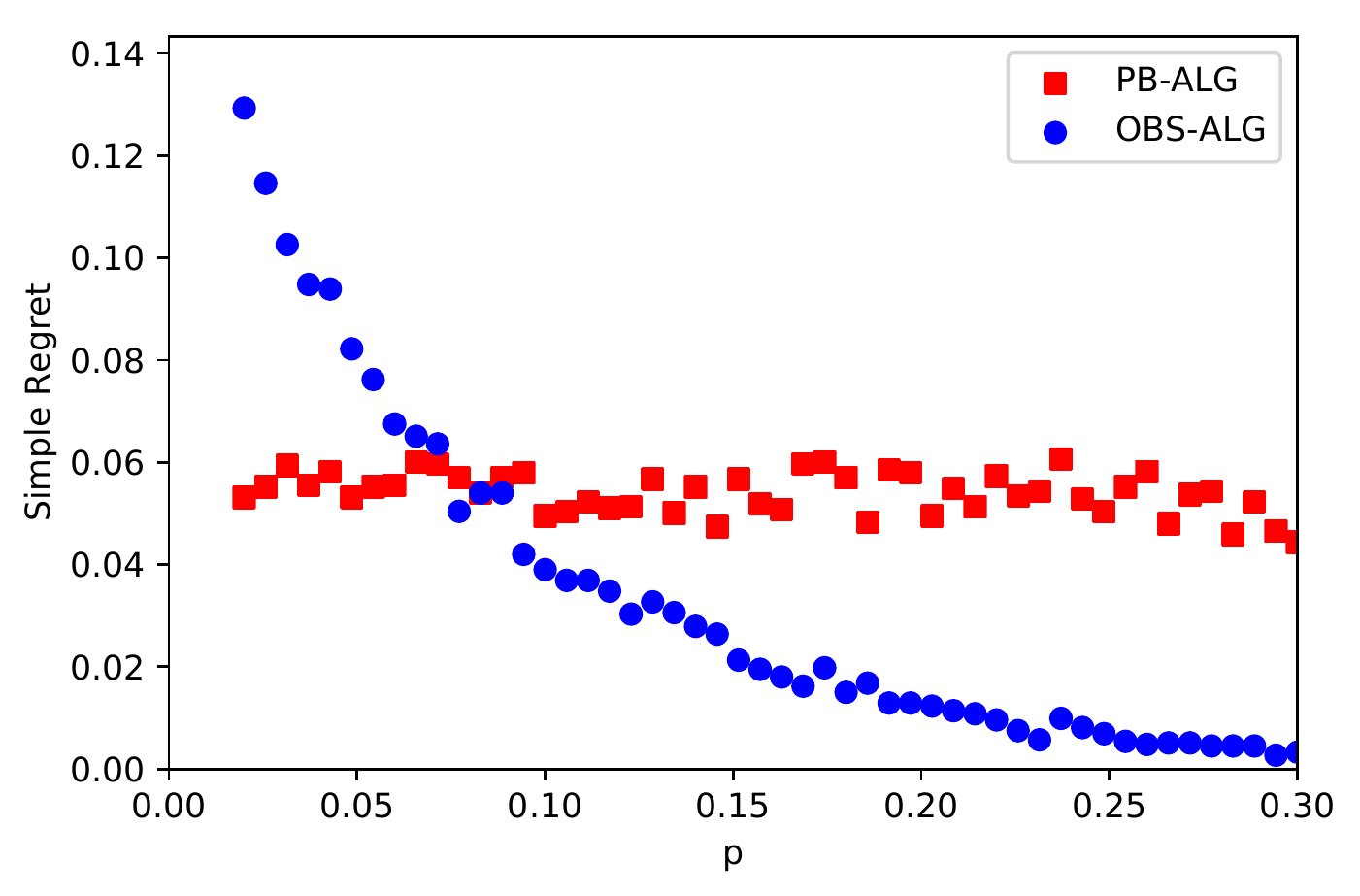}
    \caption{{\small \OBSALG vs \PBA}}
    \label{fig:regret-vs-p-obsalg-pba}
    \endminipage\hfill
\minipage{0.45\textwidth}
      \includegraphics[width=\linewidth]{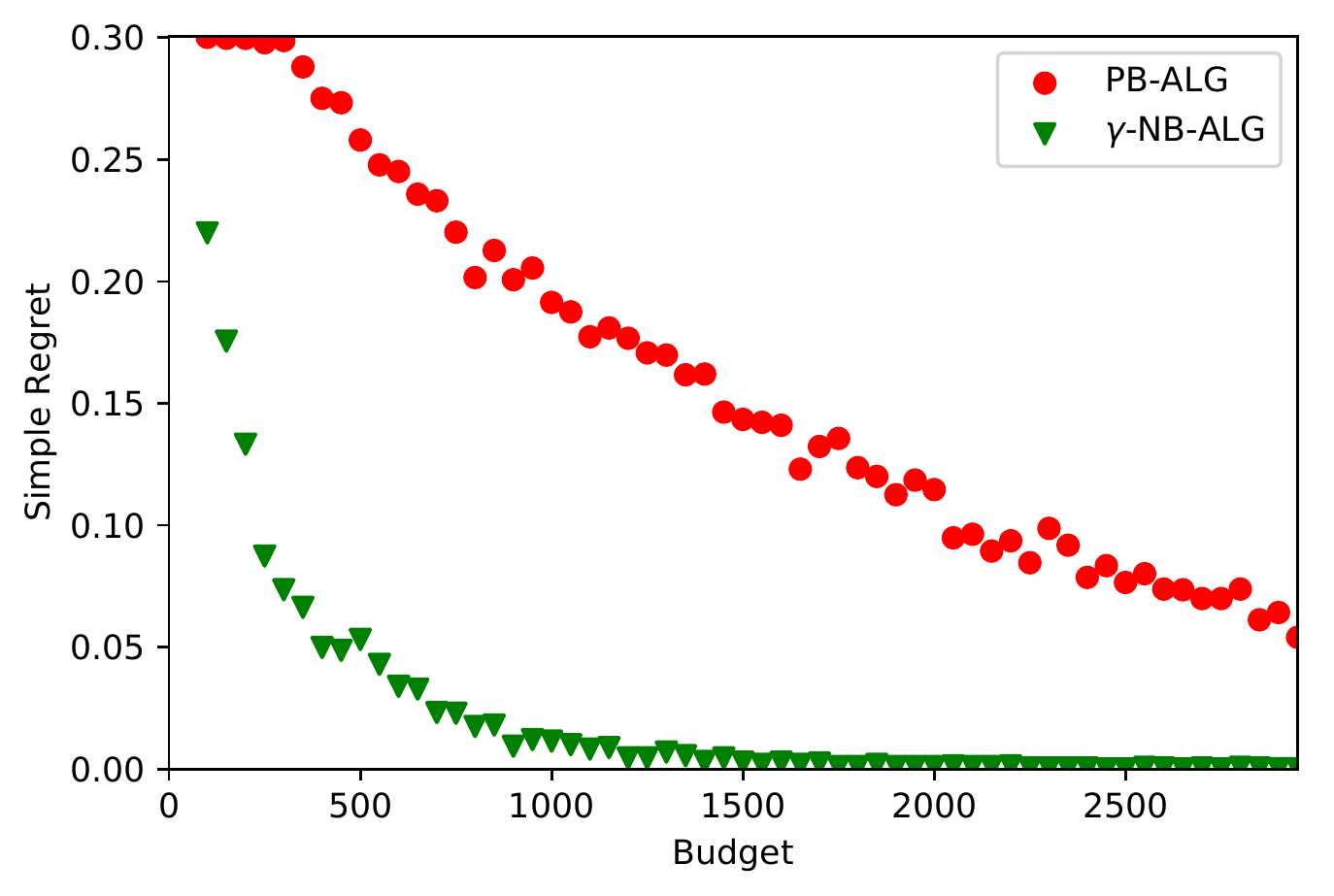}
  \caption[]%
        {{\small \GammaPBA vs. \PBA (part a)}} 	\label{fig:gammaPBA_1}   
\endminipage\hfill
\end{figure}
\noindent \textbf{Experiment $1$ (\OBSALG  vs. \PBA)}:
This experiment compares the performance of \OBSALG  with \PBA for a fixed budget on a parallel graph with $M=50$, i.e the reward variable $Y$ has $50$ parents $X_1,\ldots,X_{50}$:  $X_i \sim Bernoulli(p_i)$ for $i\in [50]$. The rewards variable $Y$ depends on $X_i$'s as follows (unknown to both the algorithms): if $X_1 =1$ then $Y \sim Bernoulli(0.5+\epsilon)$ and otherwise $Y \sim Bernoulli(0.5-\epsilon^\prime)$, where $\epsilon =0.3$, and $\epsilon^\prime = \frac{p_1\epsilon}{1-p_1} \sim 0.006$. The chosen causal graph structure is the same as in the experiments of \cite{LattimoreLR16}. Throughout the experiment $p_i = 0.5$ for $i\in [3,50]$, and $p_1 = p_2$ is the minimum probability $p$. The value of $p_1$ and $p_2$, i.e. the minimum probability $p$ is increased from $0.02$ to $0.3$. Figure \ref{fig:regret-vs-p-obsalg-pba} plots the simple regret of these algorithms with respect to minimum probability. The regret is computed by averaging it over $1000$ independent runs. The budget $B$ is fixed to a moderate value of $100$ and the cost of intervention $\gamma$ to $1$. The plot in Figure \ref{fig:regret-vs-p-obsalg-pba} shows an inverse relationship between simple regret of \OBSALG and $p$ as proved in Theorem \ref{theorem: simple regret observational algorithm}, whereas the simple regret of \PBA does not depend on $p$. Recall that the expected simple regret of \PBA depends on $m(\mathbf{p})$ (and not on $p$) and for the $p_i$'s as stated before, the quantity $m({\bf p}) = 2$, does not change.  Note that the performance of \GammaPBA is best for $\gamma=1$ and $m({\bf p}) = 2$. Finally, also observe that after a threshold value of $p$, \OBSALG starts performing much better than \PBA as can be seen from the plot.

\begin{figure}[!htb]
\centering
\minipage{0.45\textwidth}
   \includegraphics[width=\linewidth]{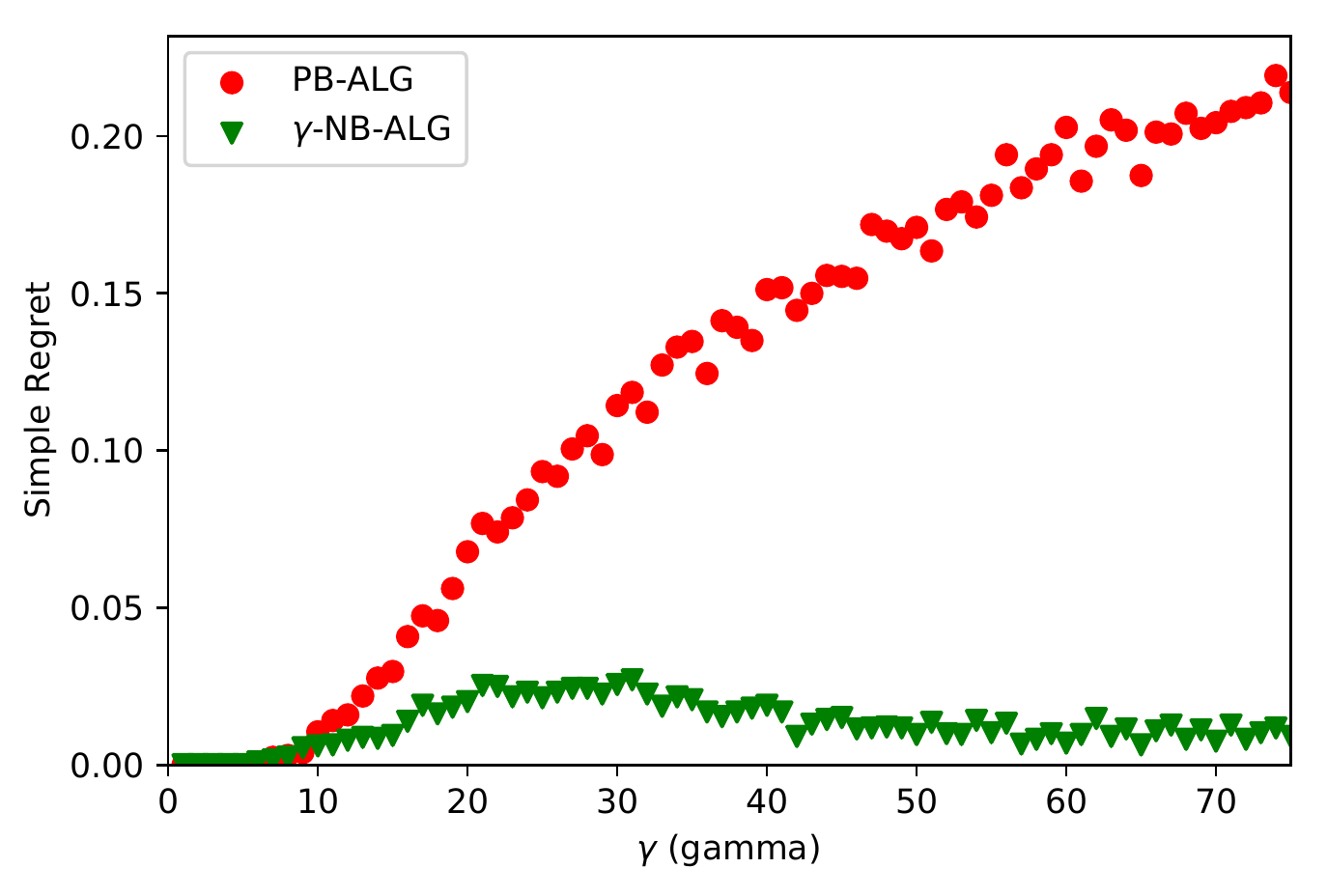}
   \caption[]%
        {{\small \GammaPBA vs. \PBA (part b)}}   \label{fig:gammaPBA_2}
\endminipage\hfill
\minipage{0.45\textwidth}
    \includegraphics[width=\linewidth]{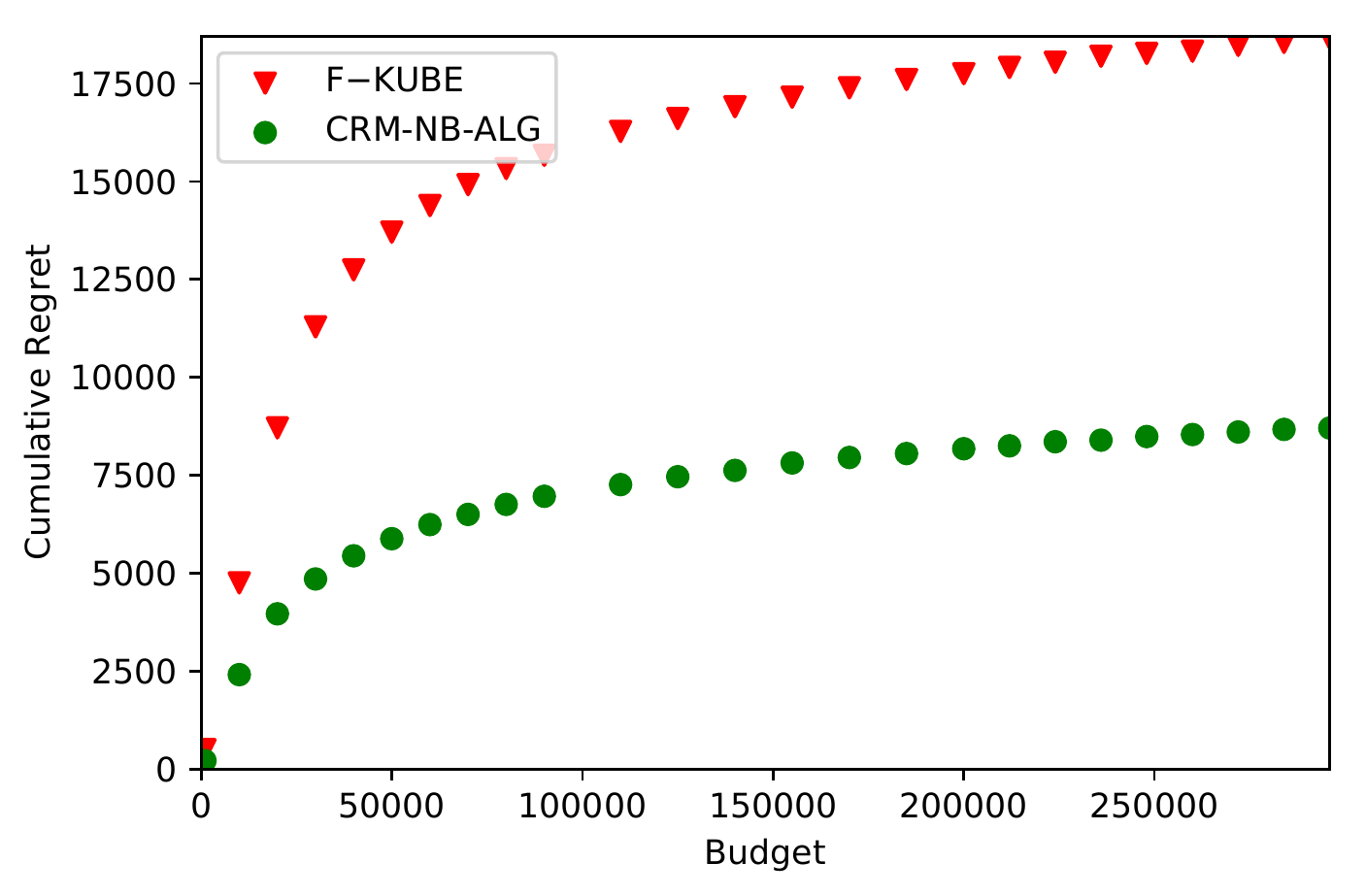}
		\caption[$\gamma = 1$]%
        {{\small $\gamma = 1$}}   \label{fig:gamma1}
\endminipage\hfill
\end{figure}
\noindent \textbf{Experiment $2$ ($\GammaPBA$ vs.\ $\PBA$)}:
This experiment compares the performance of $\GammaPBA$ and $\PBA$ on a parallel graph with $M=50$, i.e the reward variable $Y$ has $50$ parents $X_1,\ldots,X_{50}$: $X_i \sim Bernoulli(p_i)$ for $i\in [50]$, $p_1 = p_2 = 0.02$, and $p_i = 0.5$ for $i\in [3,50]$. For this choice of $p_i$'s the \PBA algorithm asymptotically achieves its best regret. The rewards variable $Y$ depends on $X_i$'s as follows (unknown to both the algorithms): if $X_1 =1$ then $Y \sim Bernoulli(0.5+\epsilon)$ and otherwise $Y \sim Bernoulli(0.5-\epsilon^\prime)$, where $\epsilon =0.3$, and $\epsilon^\prime = \frac{p_1\epsilon}{1-p_1} \sim 0.006$. Under these settings, \cite{LattimoreLR16} demonstrated a faster exponential decay of simple regret compared to the non-causal algorithms. Since in this experiment we compare \GammaPBA to \PBA (adapted to the budgeted version), we choose the same causal graph and distribution. The part a of this experiment in Figure \ref{fig:gammaPBA_1} compares the simple regret of the two algorithms when $\gamma =60$ and the budget is increased to $3000$. The regret is computed by averaging it over $1000$ independent runs. The part b of this experiment in Figure \ref{fig:gammaPBA_2} illustrates the effect on the simple regret of the algorithms as $\gamma$ increases from $1$ to $75$. In Figure \ref{fig:gammaPBA_2}, observe that till a threshold value of $\gamma$ both algorithms have very close simple regret and post the threshold, $\GammaPBA$ trades off between observations and interventions to yield a much better simple regret.
\begin{figure}[!htb]
\centering
\minipage{0.45\textwidth}
    \includegraphics[width=\linewidth]{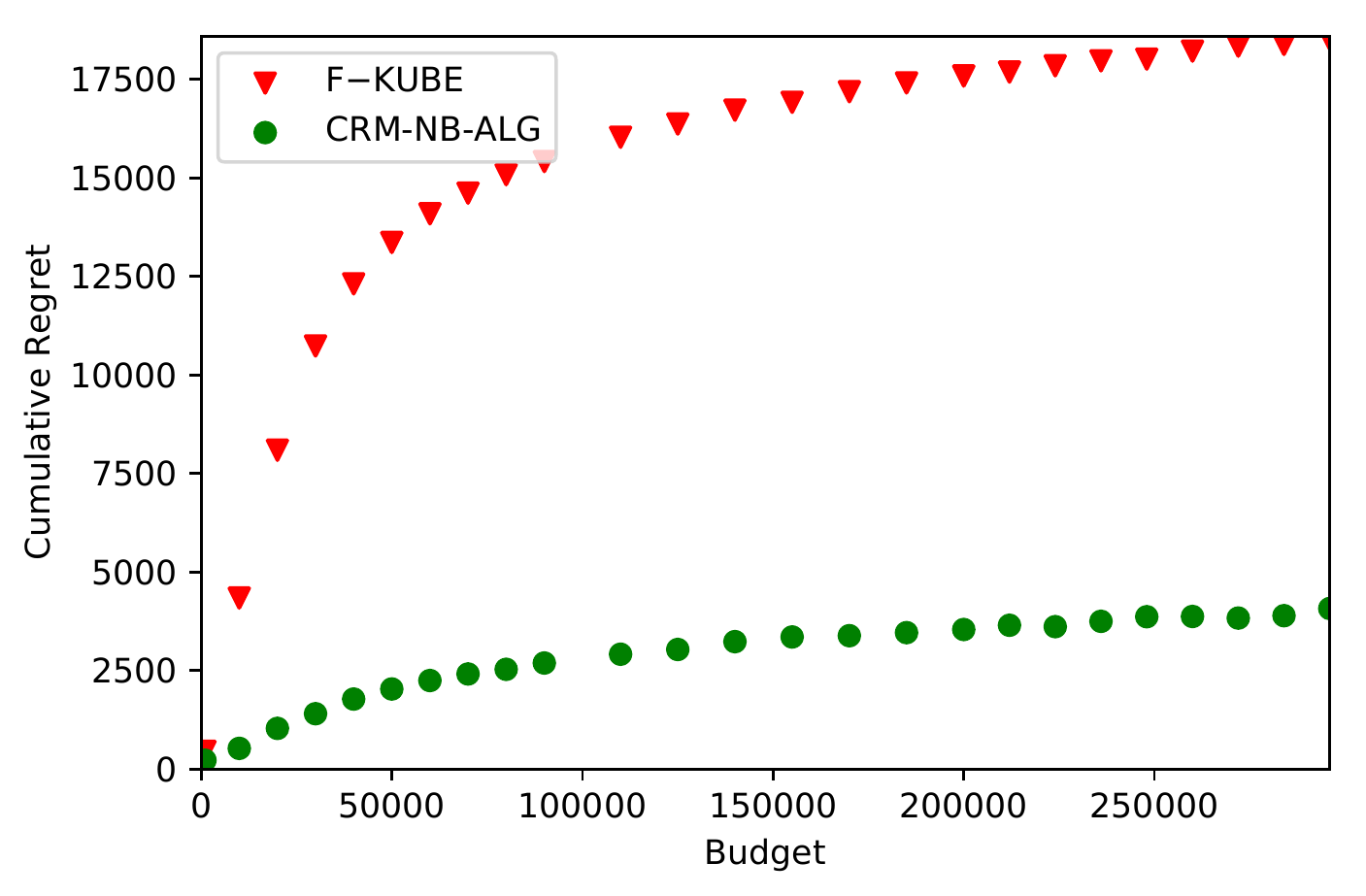}
		\caption[$\gamma = 1.1$]%
        {{\small $\gamma = 1.1$}}    \label{fig:gamma1.1}
\endminipage\hfill
\minipage{0.45\textwidth}
    \includegraphics[width=\linewidth]{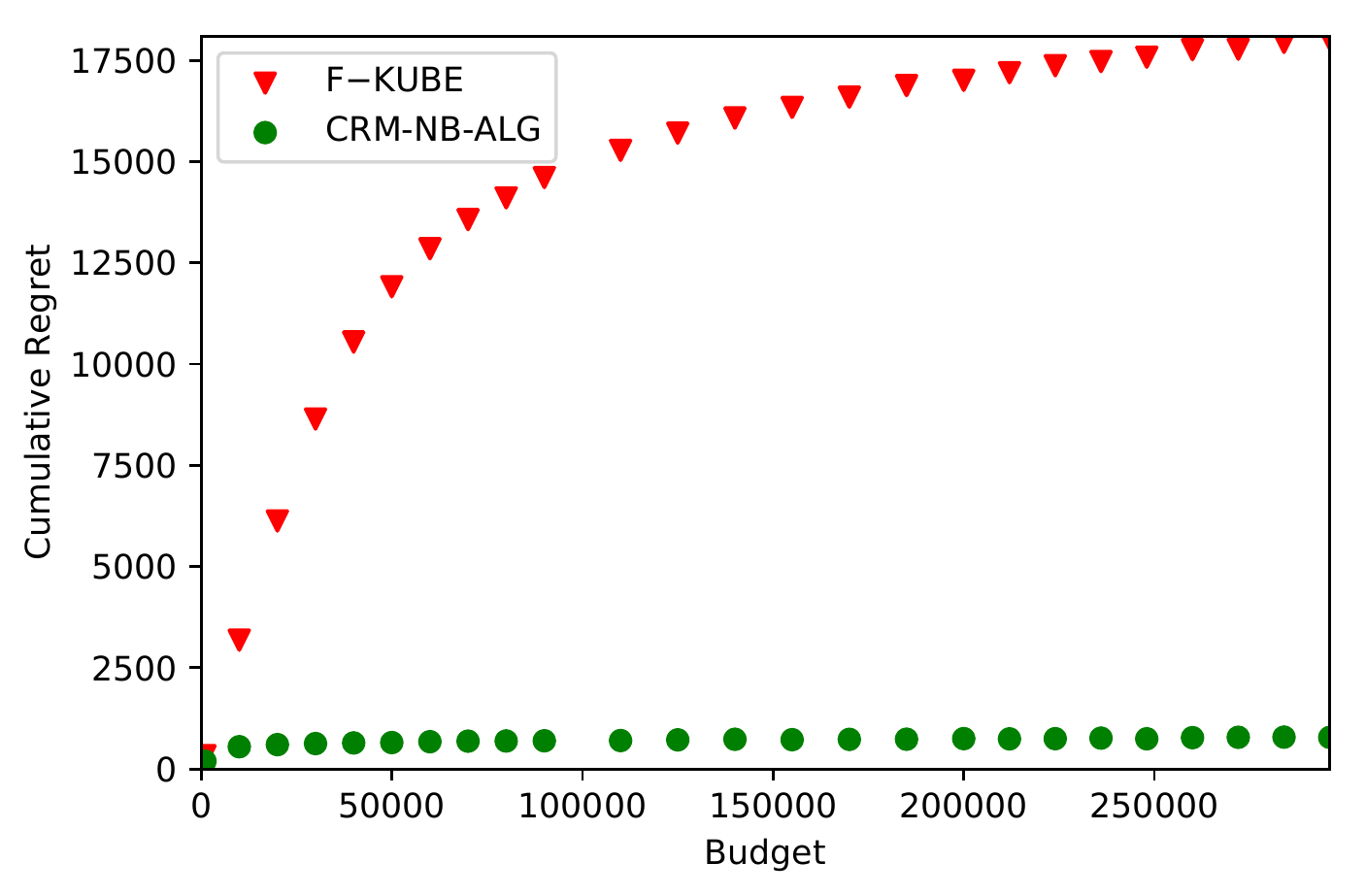}
		\caption[$\gamma = 1.5$]%
        {{\small $\gamma = 1.5$}}    
		\label{fig:gamma1.5}
\endminipage\hfill
\end{figure}

\noindent \textbf{Experiment $3$ ($\mathtt{F-KUBE}$ vs.\ $\CRMPB$)}:
This experiment compares the performance of $\mathtt{F-KUBE}$ and \CRMPB. The model is as in Experiment 2, except $\epsilon = 0.5$, i.e. the best arm has reward $1$. If the reward distribution is the same as in experiment $1$ the cumulative regret of \CRMPB even with $\gamma = 1.1$ converges very quickly to a small constant. This is attributed to the fact that the observation arm is closer to being optimal (i.e. $d_0$ is smaller). Even though this validates the better performance of our algorithm, for a better visual description we set the expected reward of the best arm to $1$. Even with this reward distribution, the performance of \CRMPB is much better than $\mathtt{F-KUBE}$. Figures \ref{fig:gamma1}, \ref{fig:gamma1.1}, and \ref{fig:gamma1.5} illustrate the cumulative regrets of both the algorithms for $\gamma$ equal to $1, 1.1$ and $1.5$ respectively as the budget is increased. The regret is computed by averaging over $50$ independent runs. Notice that \CRMPB yields a much better regret in all three cases and its regret is constant for $\gamma = 1.5$.
\begin{figure}[!htb]
\centering
\minipage{0.45\textwidth}
  \begin{center}
    \includegraphics[scale=0.5]{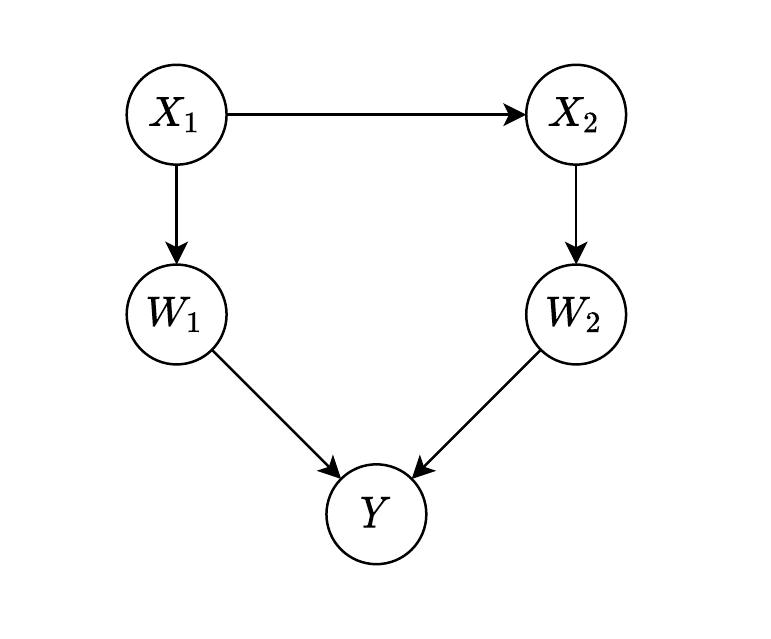}
    \caption[]{{\small General Graphs}}
    \label{fig:generalGraph}
    \end{center}
\endminipage\hfill
\minipage{0.45\textwidth}
  \includegraphics[width=\linewidth]{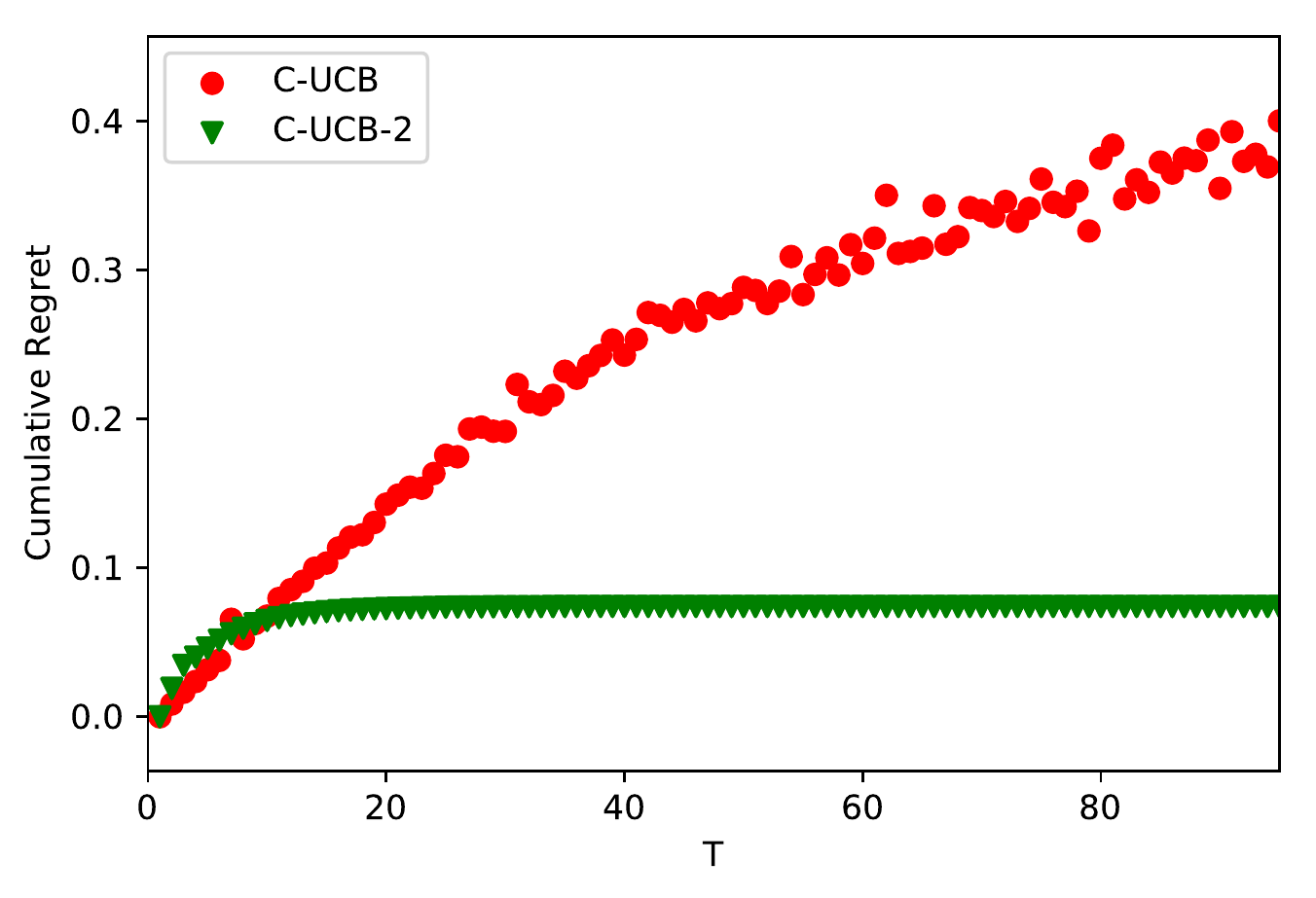}
		\caption[observationalAlgo]%
        {{\small \CUCBTwo vs. \CUCB}}   
		\label{fig:cucb2vscucb}
\endminipage\hfill
\end{figure}

\noindent \textbf{Experiment $4$ ($\CUCB$ vs.\ $\CUCBTwo$)}:
This experiment compares the performance of \CUCB and \CUCBTwo. The causal graph used in this experiment is as shown in figure \ref{fig:generalGraph}. Notice that this graph has a backdoor path from $X_2$ to $Y$ and therefore algorithms such as \PBA, \GammaPBA and \CRMPB, which are for no-backdoor graphs, cannot be used. Our conditional probabilities for nodes ($P(\text{node}|Pa(\text{node}))$) are given in Table \ref{table: conditional distribution}.
\begin{table}[ht!]
\centering
\begin{tabular}{|l|l|}
\hline
\textbf{Conditional Variable} & \textbf{Probability} \\ \hline
$X_1=0$                       & $0.45$               \\ \hline
$X_1=1$                       & $0.55$               \\ \hline
$X_2=0 | X_1 = 0$             & $0.55$               \\ \hline
$X_2=1 | X_1 = 0$             & $0.45$               \\ \hline
$X_2=0 | X_1 = 1$             & $0.45$               \\ \hline
$X_2=1 | X_1 = 1$             & $0.55$               \\ \hline
$W_1=0 | X_1 = 0$             & $0.46$               \\ \hline
$W_1=1 | X_1 = 0$             & $0.54$               \\ \hline
$W_1=0 | X_1 = 1$             & $0.54$               \\ \hline
$W_1=1 | X_1 = 1$             & $0.46$               \\ \hline
$W_2=0 | X_2 = 0$             & $0.52$               \\ \hline
$W_2=1 | X_2 = 0$             & $0.48$               \\ \hline
$W_2=0 | X_2 = 1$             & $0.48$               \\ \hline
$W_2=1 | X_2 = 1$             & $0.52$               \\ \hline
\end{tabular}
\caption{Conditional Probability Distributions}
    \label{table: conditional distribution}
\end{table}

The conditional distribution of the reward variable $Y$ was chosen as $Y|w_1,w_2 = \theta_1 X_1 + \theta_2 X_2 + \epsilon$, where $\theta_1$ and $\theta_2$ are fixed to $0.25$ (similar to that in \cite{LU2020}) and $\epsilon$ is distributed as $\mathcal{N}(0,0.01)$. Here $\mathcal{N}(0,0.01)$ denotes the normal distribution with mean $0$ and standard deviation $0.01$. The conditional probabilities in the above table are chosen to be close to each other in order to ensure that the expected rewards for all the arms are competitive and the algorithm takes longer to distinguish between them. The expected reward of the four arms $do(X_i = x), i\in [2], x\in \{0,1\}$ are given in the Table \ref{table: rewards}.
\begin{table}[ht!]
\centering
\begin{tabular}{|l|l|}
\hline
\textbf{Arm} & \textbf{Expected Reward} \\ \hline
$do(X_1=0)$  & $0.2595$                 \\ \hline
$do(X_1=1)$  & $0.2405$                 \\ \hline
$do(X_2=0)$  & $0.244$                  \\ \hline
$do(X_2=1)$  & $0.254$                  \\ \hline
\end{tabular}
\caption{Expected Reward of the Arms}\label{table: rewards}
\end{table} 

Figure \ref{fig:cucb2vscucb} shows a comparison between the cumulative regret incurred by both algorithms for values of $T$ in the range $[5,100]$. The regret is computed by averaging over $500$ independent runs. Notice that the regret of \CUCB is much higher than \CUCBTwo and also grows with time. Moreover, the regret of \CUCBTwo grows a little initially and then becomes constant as proved in Theorem \ref{theorem: cumulative regret for general graphs}.   
\section{Discussion and Future Work}
The \MAB problem can be used to model several real-world scenarios where additional information besides the reward of the pulled arms is available and hence the study of the \MAB problem with side-information has been an area of significant interest in the research community.
One of the most prominent models with side-information is the contextual \MAB problem where the algorithm receives extra information (called \emph{context}) before each arm pull \cite{LU2010contextual}. A class of bandit problems where the side-information obtained conforms to a feedback graph has also been studied in the literature \cite{ALON2015}. The special case of parallel causal graphs studied in \cite{LattimoreLR16} and in this work is in fact captured by such a model, but as shown by \cite{LattimoreLR16} their regret bounds are not optimal in this setting.

In this work, we study the the causal bandit problem for no-backdoor graphs in the budgeted bandit framework. In this setting, observations are cheaper compared to interventions, which is practically well-motivated. In Sections \ref{sec: simple regret} and \ref{section: cum regret for parallel bandits} we provided two algorithms, \GammaPBA and \CRMPB, that minimized the expected simple regret and expected cumulative regret respectively. \cite{SenSDS17} also studies the best intervention identification problem via importance sampling under budget constraint. But in contrast to our work, they consider soft interventions on a single node $V$, and also assume that the interventional distributions and the marginals of the parent distribution of the node $V$ are known. This is incomparable with hard interventions on no-backdoor graphs, where interventions can be performed on different variables and the parent distributions of the intervened nodes are not known. Also their setting is parameterized by $B'$ and $T$, where $B'$ is the upper bound on the average cost of sampling and $T$ is the total number of samples that the algorithm draws. This can be mapped to our setting by setting the budget to be $B'T$. In our budgeted setting $T$ is not given as input to the algorithm, and this is important for the trade-off between observations and interventions.

In the non-budgeted setting, we showed that our algorithm \CUCBTwo has constant expected cumulative regret in terms of instance-parameters. We conjecture that the worst-case regret bound of our algorithm matches that in \cite{LU2020}, and resolving that remains open. The work by \cite{sachidananda2017} studies a similar problem as that in our work and experimentally show the effectiveness of Thompson Sampling but do not provide any theoretical guarantees.

Finally, many of the works in the literature such as those of \cite{LU2020} and \cite{LattimoreLR16} assume that the parent distribution for each intervention is known to the algorithm. We only make this assumption in Section \ref{section: cumulative regret for general graphs}. This assumption is limiting in practice and showing a non-trivial regret guarantee for settings without this assumption remains an important open direction.
\section{Proofs of Theorems}
\subsection{Theoretical Preliminaries}
We require the following two versions of the Chernoff-Hoeffeding inequality in our proof.
\begin{lemma}[Chernoff-Hoeffeding inequality]\label{lemma: chernoff-hoeffeding inequality}
Suppose $X_1, \ldots, X_T$ are independent random variables taking values in the interval $[0,1]$, and let $X = \sum_{t\in [T]} X_t$ and $\overline{X} = \frac{\sum_{t\in [T]}X_t}{T}$. Then for any $\varepsilon \geq 0$ the following holds:
$$a)~~ \mathbb{P}\{X - E[X] \geq \varepsilon \} \leq e^{\frac{-2\varepsilon^2}{T}}, $$
$$b)~~ \mathbb{P}\{\overline{X} - E[\overline{X}] \geq \varepsilon \} \leq e^{-2\varepsilon^2T}~. $$
\end{lemma}

\subsection{Proof of Theorem \ref{theorem: simple regret observational algorithm}}\label{secappendix: proof of obs simple regret}
For $B\in \N$, let $\varepsilon = \sqrt{\frac{2}{pB}\log (16pMB)}$. Also, let $L = \min_{t\in \N}\{2\sqrt{\frac{1}{t}\log 16 pMt} \leq p\}$. Note that $L$ is a \emph{finite} constant dependent on $p$ and $M$, and that for all $B\geq L$
\begin{equation}\label{equation: bound on epsilon for proof of theorem 1}
\varepsilon \leq \sqrt{p/2}  ~. 
\end{equation}
In this proof, $i$ indexes the set $[M]$, and $x$ indexes the set $\{0,1\}$. Recall $p_{i,x} = \mathbb{P}\{X_i=x\}$ and $p_i = p_{i,1}$. Also note that \OBSALG plays the arm $a_0$ for $B$ rounds. For $i\in [M]$, let $X_i(t)$ be the value of $X_i$ sampled in round $t \in [B]$. For all $(i,x)$, let 
$$\widehat{p}_{i,x} = \frac{\sum_{t\in [B]} \mathbb{1}\{X_i(t) = x\}}{B}~,~~ \text{and}$$
$$\widehat{\mu}_{i,x} = \frac{\sum_{t\in [B]} Y_t\cdot \mathbb{1}\{X_i(t)=x\}}{\sum_{t\in [B]} \mathbb{1}\{X_i(t)=x\}}~,$$
where $Y_t$ is value of $Y$ sampled in round $t$. Notice that $\widehat{\mu}_{i,x}$ is the empirical estimate of $\mu_{i,x}$ computed by \OBSALG at the end of $B$ rounds. Similarly the empirical estimate of $\mu_0$, denoted $\widehat{\mu}_0$, is computed by \OBSALG at the end of $B$ rounds as follows:
$$\widehat{\mu}_{0} = \frac{\sum_{t\in [B]} Y_t}{B}~.$$
Finally, also let $\widehat{p}_{i} = \widehat{p}_{i,1}$. The proof of the theorem is completed using the following lemma.
\begin{lemma}\label{lemma: theorem1 bounds at the end of B rounds}
At the end of $B$ rounds played by \OBSALG the following hold:
\begin{flalign*}
    1.~~~& \mathbb{P}\{|\widehat{\mu}_0 - \mu_0| \geq \varepsilon\} \leq 2e^{-2\varepsilon^2B} \leq 4e^{-\varepsilon^2pB}~,\\
         2.~~~&  \text{For any fixed } (i,x)~~~  \mathbb{P}\Big\{\widehat{p}_{i,x}B \leq \frac{pB}{2} \Big\} \leq 2e^{-\varepsilon^2pB}~,\\
        3.~~~&    \text{For any fixed } (i,x)~~~ \mathbb{P}\{|\widehat{\mu}_{i,x} - \mu_{i,x}| \geq \varepsilon\} \leq 4e^{-\varepsilon^2pB}~.
    \end{flalign*}
\end{lemma}
\begin{proof}
1) Part 1 directly follows from Lemma \ref{lemma: chernoff-hoeffeding inequality}.\\

2) Observe that $E[\widehat{p}_i] = p_i$, and hence from Lemma \ref{lemma: chernoff-hoeffeding inequality}, for an $i\in [M]$ at the end of $B$ rounds we have
\begin{equation}\label{equation: constraint on p_i}
    \mathbb{P}\left\{|(\widehat{p}_i-p_i)B|\geq \varepsilon B\sqrt{\frac{p}{2}} \right\} \leq 2e^{-\varepsilon^2pB}~.
\end{equation}
Since $\varepsilon \leq \sqrt{p/2}$ (from Equation \ref{equation: bound on epsilon for proof of theorem 1}), $\varepsilon B\sqrt{\frac{p}{2}} \leq \frac{pB}{2}$. This implies 
\begin{equation}\label{equation: constraint on pB}
    \frac{pB}{2} \leq pB - \varepsilon B\sqrt{\frac{p}{2}}~.
\end{equation}
Hence from Equations \ref{equation: constraint on p_i} and \ref{equation: constraint on pB}, for a fixed $(i,x)$ the following holds:
\begin{equation*}
    \mathbb{P}\Big\{\widehat{p}_{i,x}B \leq \frac{pB}{2} \Big\} \leq 2e^{-\varepsilon^2pB}~.
\end{equation*}

3) Notice that $\widehat{p}_{i,x}B$ is the number of times $X_i$ was sampled as $x$ in $B$ rounds. In particular, part 2 of Lemma \ref{lemma: theorem1 bounds at the end of B rounds} bounds the probability that the number of times $X_i$  was sampled as $x$ is small. We use this to prove part 3.
First observe that from Lemma \ref{lemma: chernoff-hoeffeding inequality} we have
\begin{equation}\label{equation: proof of theorem 1 bounding empirical estimate of i,x conditioned on enough samples}
    \mathbb{P}\Big\{|\widehat{\mu}_{i,x} - \mu_{i,x}| \geq \varepsilon \Big| \widehat{p}_{i,x}B > \frac{pB}{2} \Big\} \leq 2e^{-\varepsilon^2pB}~.
\end{equation}
In particular, Equation \ref{equation: proof of theorem 1 bounding empirical estimate of i,x conditioned on enough samples} bounds the error probability of estimating $\widehat{\mu}_{i,x}$ conditioned on the event that $X_i$ has been sampled as $x$ sufficiently many times. Next by law of total probability, for any fixed $(i,x)$,
\begin{align*}\label{equation: bound on mu_ix}
   \mathbb{P}\{|\widehat{\mu}_{i,x} - \mu_{i,x}| \geq \varepsilon\} &=  \mathbb{P}\Big\{|\widehat{\mu}_{i,x} - \mu_{i,x} | \geq \varepsilon \Big| \widehat{p}_{i,x}B >\frac{pB}{2} \Big\}\cdot \mathbb{P}\Big\{\widehat{p}_{i,x}B > \frac{pB}{2} \Big\} \\
   &~~~~+  \mathbb{P}\Big\{|\widehat{\mu}_{i,x} - \mu_{i,x}| \geq \varepsilon \Big| \widehat{p}_{i,x}B \leq \frac{pB}{2} \Big\}\cdot \mathbb{P}\Big\{\widehat{p}_{i,x}B \leq \frac{pB}{2} \Big\}\\
    \mathbb{P}\{|\widehat{\mu}_{i,x} - \mu_{i,x}| \geq \varepsilon\} &\leq  \mathbb{P}\Big\{|\widehat{\mu}_{i,x} - \mu_{i,x}| \geq \varepsilon \Big| \widehat{p}_{i,x}B > \frac{pB}{2} \Big\} +  \mathbb{P}\Big\{\widehat{p}_{i,x}B \leq \frac{pB}{2} \Big\}~.
\end{align*}
Hence, from Equation \ref{equation: proof of theorem 1 bounding empirical estimate of i,x conditioned on enough samples} and part 2 of Lemma \ref{lemma: theorem1 bounds at the end of B rounds} we have
$$\mathbb{P}\{|\widehat{\mu}_{i,x} - \mu_{i,x}| \geq \varepsilon\} \leq 4e^{-\varepsilon^2pB}~. $$
\end{proof}
Let $U_0$ be the event that $|\widehat{\mu}_0 - \mu_0| \leq \varepsilon$, and for any $i,x$ let $U_{i,x}$ be the event $|\widehat{\mu}_{i,x} - \mu_{i,x}| \leq \varepsilon$. Also let $U = (\cap_{i,x} U_{i,x}) \cap U_0$, $\overline{U}$ denote the compliment of $U$. Then applying union bound on the events in part 1 and 3 in Lemma \ref{lemma: theorem1 bounds at the end of B rounds}, we have that 
$$\mathbb{P}\{\overline{U}\} \leq (2M+1) \cdot 4e^{-\varepsilon^2pB}$$
Hence, we have that
$$\mathbb{P}\{U\} \geq 1 - (8M+4)e^{-\varepsilon^2pB} \geq 1- 16M e^{-\varepsilon^2pB}.$$
Let $a^{*} = \arg\max_{a\in\mathcal{A}}(\mu_{a})$. Note that if event $\overline{U}$ holds then the simple regret of \OBSALG, $r_{\text{\OBSALG}}(B) \leq 1$. On the other hand, if the event $U$ holds, and $a_B$ is the arm output by the algorithm, then $r_{\text{\OBSALG}}(B) = \mu_{a^*} - \mu_{a_B} \leq 2\varepsilon$. Setting $\delta = 16M e^{-\varepsilon^2pB}$, and substituting the value of $\varepsilon$, we have $\delta = \frac{1}{16Mp^2B^2}$. Hence, the expected simple regret is at most: 
\begin{equation}\label{equation: regret of obsalg}
\delta + \sqrt{\frac{8}{pB}\log (16pMB)} =  \frac{1}{16Mp^2B^2} + \sqrt{\frac{8}{pB}\log (16pMB)} = O\Bigg( \sqrt{\frac{1}{pB}\log (pMB)}\Bigg)~.
\end{equation}
\subsection{Proof of Theorem \ref{theorem: simple regret of gammapba}}\label{secappendix: proof of gamma simple regret}
For convenience, we denote $m(\mathbf{p})$ and $m(\widehat{\mathbf{p}})$ as $m$ and $\widehat{m}$ respectively. 
Throughout the proof we assume that $B$ is such that: a) $B\geq \max(\gamma m, pM)$ and b) $B \geq \max(\frac{16}{p^2}\log \frac{2MB}{\gamma m}, \frac{16}{p^2}\log 2pMB)$. 
Note that the two constraints hold for sufficiently large $B$. To begin with observe that if $\gamma = \theta(\frac{1}{p\cdot m(\mathbf{p})})$ then $O\left(\sqrt{\frac{1}{pB}\log (pMB)}\right) = O\left(\sqrt{\frac{\gamma m}{B}\log \frac{MB}{\gamma m}}\right)$. Hence, it is sufficient to show that if $\gamma \leq \frac{1}{5p\cdot m(\mathbf{p})}$ then the expected simple regret of \GammaPBA is $O\left(\sqrt{\frac{\gamma m}{B}\log \frac{MB}{\gamma m}}\right)$ and if $\gamma \geq \frac{5}{p\cdot m(\mathbf{p})}$ then the expected simple regret of \GammaPBA is $O\left(\sqrt{\frac{1}{pB}\log (pMB)}\right)$. 
Theorem \ref{theorem: simple regret of gammapba} is proved using Lemmas \ref{lemma: bounds on p_i} and \ref{lemma: bounding m}.
\begin{lemma}\label{lemma: bounds on p_i}
Let  $\widehat{p}_{i,1} = \widehat{p}_i$ and $F = \mathbb{1}\{\text{At the end of } B/2 \text{ rounds there is an } i\in [M] \text{ such that } |\widehat{p}_i - p_i| \geq \frac{p}{4}\}$. Then $\mathbb{P}\{F=1\} \leq 2Me^{-\frac{p^2}{16}B}$.
\end{lemma}
\begin{proof}
Let $F_i = \mathbb{1}\{\text{At the end of } B/2 \text{ rounds } |\widehat{p}_i - p_i| \geq \frac{p}{4}\}$. Then from Lemma \ref{lemma: chernoff-hoeffeding inequality}, 
$$\mathbb{P}\{F_i =1\} \leq 2e^{-\frac{p^2}{16}B}~.$$
Taking union bound over $F_i = 1$ for $i\in [M]$, we have $\mathbb{P}\{F=1\} \leq 2Me^{-\frac{p^2}{16}B}$.
\end{proof}
The following lemma is similar to Lemma 8 in \cite{LattimoreLR16}.
\begin{lemma}\label{lemma: bounding m}
Let $F$ be as in Lemma \ref{lemma: bounds on p_i}, and let $I = \mathbb{1}\{\text{At the end of } B/2 \text{ rounds }~\frac{2m(\mathbf{p})}{5} \leq m(\widehat{\mathbf{p}}) \leq 2m(\mathbf{p})\}$. Then $F = 0$ implies $I=1$, and in particular, 
$\mathbb{P}\{I=1\} \geq 1 - 2Me^{-\frac{p^2}{16}B}.$
\end{lemma}
\begin{proof}
We are interested in the quantity $\min_{x\in\{0,1\}}p_{i,x}$ for each $i \in [M]$. Without loss of generality, let us assume $\min_{x\in\{0,1\}}p_{i,x} = p_{i,1} = p_i$ for each $i\in[M]$, and also $p_1 \leq p_2 \leq \ldots \leq p_M \leq \frac{1}{2}$. 
Note that $F=0$ implies after $B/2$ rounds for all $i\in [M]$~ $|\widehat{p}_i - p_i| \leq \frac{p}{4}$. Now, from the definition of $m(\mathbf{p})$ we know that there is an $\ell \leq m$ such that the following is true: for $i > \ell$, $p_i \geq  \frac{1}{m}$. Further, we can also conclude that $p \leq \frac{1}{m-1}$ (otherwise $m(\mathbf{p}) = m-1$). Hence, $\widehat{p}_i \geq p_i -\frac{p}{4} \geq \frac{1}{m} - \frac{1}{4(m-1)}$
. Hence for $i > \ell$, $\widehat{p}_i \geq \frac{3m-4}{4m(m-1)} \geq \frac{1}{2m}$ (since $m \geq 2$). Since $\ell \leq m$, we have $|\{j \mid \widehat{p}_j < \frac{1}{2m}\}| \leq 2m$. This implies $\widehat{m} \leq 2m$.
To prove the other inequality, observe that for each $i\leq m$, we have $p_i \leq \frac{1}{m-1}$ (otherwise, $m(\mathbf{p}) \leq m - 1$). Then, $\widehat{p}_i \leq p_i +\frac{p}{4} \leq \frac{1}{m-1} + \frac{1}{4(m-1)} \leq \frac{5}{4(m-1)} \leq \frac{5}{2m}$. Hence for $i\leq m$, $\widehat{p}_i \leq \frac{5}{2m}$. This implies $\widehat{m} \geq \frac{2m}{5}$.
\end{proof}
From Lemmas \ref{lemma: bounds on p_i} and \ref{lemma: bounding m} it follows that if $F=0$ then at the end of $B/2$ rounds the following holds:
$$\frac{p}{2} \leq \widehat{p} \leq \frac{3p}{2} ~~~\text{and} ~~~\frac{2m}{5} \leq \widehat{m} \leq 2m$$ 
This implies that if $F=0$ then at then end of $B/2$ rounds the following holds:
\begin{equation}\label{equation: inequality on mp}
    \frac{p\cdot m}{5} \leq \widehat{p}\cdot \widehat{m} \leq 5p\cdot m
\end{equation} 
\textbf{Case a} ($\gamma < \frac{1}{5p\cdot m}$): We condition on $F=0$. Hence, from the argument above it follows that Equation \ref{equation: inequality on mp} holds. Hence, $\gamma < \frac{1}{5 p\cdot m} \leq \frac{1}{\widehat{p}\cdot \widehat{m}}$. This implies at step 6 in \GammaPBA, $ \widehat{p}\cdot \widehat{m} < \frac{1}{\gamma}$, and \GammaPBA executes steps 11-14. That is \GammaPBA makes $\frac{B}{4\gamma}$ interventions in the remaining rounds. The algorithm constructs set $A = \{a_{i,x} \mid \widehat{p}_{i,x} \leq \frac{1}{\widehat{m}}\}$. Now for arms in $A$, $\widehat{\mu}_{i,x}$ is computed as in step 14 of \Gammapba, i.e for $a_{i,x} \in A$ 
$$\widehat{\mu}_{i,x} = \frac{2\gamma |A|}{B}\sum_{t=B/2+1}^{B/2\gamma} Y_t\cdot \mathbb{1}\{a_t = a_{i,x}\}~.$$
Notice that $|A| \leq \widehat{m}$ (from the definition of $m(\widehat{\mathbf{p}})$). Hence
$$\frac{B}{2\gamma\cdot |A|} \geq \frac{B}{2\gamma\cdot \widehat{m}} \geq \frac{B}{4\gamma\cdot m}~~~~~~~~~~\text{(from Lemma \ref{lemma: bounding m}}).$$ 
Thus from Lemma \ref{lemma: chernoff-hoeffeding inequality} for each arm $a_{i,x}\in A$ and any $\varepsilon > 0$ 
\begin{equation}\label{equation: probability of arms in A}
    \mathbb{P}\Big\{|\widehat{\mu}_{i,x} -  \mu_{i,x}| \geq \varepsilon \Big| F=0\Big\} \leq 2e^{-\varepsilon^2\frac{B}{2 \gamma m}}  
\end{equation}
Also for arms not in $A$, $\widehat{\mu}_{i,x}$ is computed as in step 3 of \Gammapba, i.e. for $a_{i,x} \notin A$
$$\widehat{\mu}_{i,x} = \frac{\sum_{t=1}^{B/2}Y_t\cdot \mathbb{1}\{X_i = x\}}{\sum_{t=1}^{B/2}\mathbb{1}\{X_i = x\}}~.$$
Moreover, if $a_{i,x} \notin A$ then $\widehat{p}_{i,x} \geq \frac{1}{\widehat{m}} \geq \frac{1}{2m}$. Since $\widehat{p}_{i,x} = \frac{2}{B}\sum_{t=1}^{B/2}\mathbb{1}\{X_i = x\}$, this implies if $a_{i,x} \notin A$ then $\sum_{t=1}^{B/2}\mathbb{1}\{X_i = x\} \geq \frac{B}{4m}$. Hence from Lemma \ref{lemma: chernoff-hoeffeding inequality}, for each arm $a_{i,x} \notin A$ and any $\varepsilon > 0$, 
\begin{equation}\label{equation: probability of arms not in A}
\mathbb{P}\Big\{|\widehat{\mu}_{i,x} -  \mu_{i,x}| \geq \varepsilon \Big| F=0 \Big\} \leq 2e^{-\varepsilon^2\frac{B}{2 m}}  \leq 2e^{-\varepsilon^2\frac{B}{2 \gamma m}} 
\end{equation}
The last inequality holds since $\gamma\geq 1$. Using Equations \ref{equation: probability of arms in A} and \ref{equation: probability of arms not in A} we have for any arm $a \in\mathcal{A}$,
$$\mathbb{P}\Big\{|\widehat{\mu}_a -  \mu_a| \geq \varepsilon \Big| F=0\Big\} \leq 2e^{-\varepsilon^2\frac{B}{2 \gamma m}} ~. $$
Hence, applying union bound we have
$$\mathbb{P}\Big\{\text{there is an } a\in \mathcal{A} \text{ such that } |\widehat{\mu}_a -  \mu_a| \geq \varepsilon \Big| F=0 \Big\} \leq (4M+2)e^{-\varepsilon^2\frac{B}{2 \gamma m}} \leq 8Me^{-\varepsilon^2\frac{B}{2 \gamma m}} ~. $$
Substituting $\varepsilon = \sqrt{\frac{8\gamma m}{B}\log \frac{MB}{\gamma m}}$ we have 
\begin{equation}\label{equation: simple regret of gamma pba for interventions and conditioned on F}
    E[r_{\text{\Gammapba}}(B)|F=0] \leq  \sqrt{\frac{8\gamma m}{B}\log \frac{MB}{\gamma m}} + \frac{8}{M^3} \left(\frac{\gamma m}{B}\right)^4 \leq \sqrt{\frac{32\gamma m}{B}\log \frac{MB}{\gamma m}} ~.
\end{equation}
To get the last inequality, we use that $\frac{8}{M^3} \left(\frac{\gamma m}{B}\right)^4 \leq \sqrt{\frac{8\gamma m}{B}\log \frac{MB}{\gamma m}}$, as $M\geq 1$ and $B\geq \gamma m$. Finally, we use Equation \ref{equation: simple regret of gamma pba for interventions and conditioned on F} and Lemma \ref{lemma: bounds on p_i} to bound the expected simple regret of $\GammaPBA$ in this case as follows: 
\begin{align*}
    E[r_{\text{\Gammapba}}(B)] &= E[r_{\text{\Gammapba}}(B)|Y=0]Pr\{Y=0\} + E[r(B)|Y=1]Pr\{Y=1\} \\
    &\leq E[r_{\text{\Gammapba}}(B)|Y=0] + Pr\{Y=1\}\ \\
    &\leq \sqrt{\frac{32\gamma m}{B}\log \frac{MB}{\gamma m}} + 2Me^{-\frac{p^2}{16}B}\\
    & = O\Bigg(\sqrt{\frac{\gamma m}{B}\log \frac{MB}{\gamma m}}\Bigg)~.
\end{align*}
In last but one line of the above equation, we use that $B$ satisfies $B \geq \frac{4}{p^2}\log \frac{2MB}{\gamma m}$ and $B\geq \gamma m$ implying $2Me^{-\frac{p^2}{16}B}$ is at most $\sqrt{\frac{32\gamma m}{B}\log \frac{MB}{\gamma m}}$.

\textbf{Case b} ($\gamma \geq \frac{5}{p\cdot m(\mathbf{p})}$): Again we condition on $F=0$, and hence Equation \ref{equation: inequality on mp} holds. Hence, $\gamma \geq \frac{5}{p\cdot m(\mathbf{p})} \geq \frac{1}{\widehat{p}\cdot m(\widehat{\mathbf{p}})}$. This implies at step 6 in \GammaPBA, $ \widehat{p}\cdot m(\widehat{\mathbf{p}}) \geq\frac{1}{\gamma}$, and \GammaPBA executes steps 7-9. That is it plays the arm $a_0$ for $B$ rounds. Thus, from the analysis of Theorem \ref{theorem: simple regret observational algorithm} we have that (see Equation \ref{equation: regret of obsalg})
\begin{equation}\label{equations: regret of gamma pba for observations conditioned on y}
    E[r_{\text{\Gammapba}}(B)|Y=0] \leq  \sqrt{\frac{1}{pB}} + \sqrt{\frac{8}{pB}\log (16pMB)}~.
\end{equation}
We use Equation \ref{equations: regret of gamma pba for observations conditioned on y} and Lemma \ref{lemma: bounds on p_i} to bound the expected simple regret of $\GammaPBA$ in this case as follows:
\begin{align*}
    E[r_{\text{\Gammapba}}(B)] &= E[r_{\text{\Gammapba}}(B)|Y=0]Pr\{Y=0\} + E[r_{\text{\Gammapba}}(B)|Y=1]Pr\{Y=1\} \\
    &\leq E[r_{\text{\Gammapba}}(B)|Y=0] + Pr\{Y=1\}\ \\
    &\leq \sqrt{\frac{1}{pB}} + \sqrt{\frac{8}{pB}\log (16pMB)} + 2Me^{-\frac{p^2}{16}B}\\
    & = O\Bigg(\sqrt{\frac{1}{pB}\log (16pMB)}\Bigg)
\end{align*}
Again in the last but one line of the above equation, we use that $\frac{4\log MB}{p^2B} \leq 1$ and hence $2Me^{-\frac{p^2}{16}B}$ is at most $\sqrt{\frac{8}{pB}\log (16pMB)}$.
\subsection{Proof of Theorem \ref{theorem: cumulative regret of gamma paralllel bandit}}\label{secappendix: proof of gamma cumulative regret}

The proof of Theorem \ref{theorem: cumulative regret of gamma paralllel bandit} requires the the following lemmas.
\begin{lemma}\label{lemma: theorem 3 bounding mu's}
For any $T\in \N$, at the end of $T$ rounds the following hold:
\begin{enumerate}
    \item $\mathbb{P} \Big\{|\widehat{\mu}_0(T) - \mu_0| \geq \frac{d_0}{4} \Big\} \leq  \frac{2}{T^{\frac{d_0^2}{8}}}$~,
    \item Let $\widehat{p}_{i,x} = \frac{\sum_{t=1}^T \mathbb{1}\{a_t =a_0 ~~\text{and}~~ X_i=x\}}{N^{0}_T}$. Then
    $\mathbb{P}\{ \widehat{p}_{i,x} \geq \frac{p}{2} \} \geq  1- \frac{1}{T^{\frac{p^2}{2}}}$~,
    \item $\mathbb{P}\Big\{ \Big|\frac{\widehat{\mu}_{i,x}(T)}{\gamma} - \frac{\mu_{i,x}}{\gamma} \Big| \geq \frac{d_0}{4} \Big\} \leq \frac{2}{T^{\frac{d_0^2 p\gamma^2}{16}}} + \frac{1}{T^{\frac{p^2}{2}}}~.$
\end{enumerate}
\end{lemma}
\begin{proof}
1. Since $\beta\geq 1$, at the end of $T$ rounds arm $a_0$ is pulled by \CRMPB at least $(\ln T)^2$ times. Hence, $N^0_T \geq (\ln T)$, and from Lemma \ref{lemma: chernoff-hoeffeding inequality} we have
\begin{equation}\label{equation: empirical estimate of mu_0 is close}
\mathbb{P}\Big\{|\widehat{\mu}_0(T) - \mu_0| \geq \frac{d_0}{4} \Big\} \leq 2e^{-\frac{d_0^2}{8}\ln T} = \frac{2}{T^{\frac{d_0^2}{8}}}
\end{equation}
2. Observe that for any $(i,x)$, $p_{i,x} \geq p$, and $\mathbb{E}[\widehat{p}_{i,x}] = p_{i,x}$. 
Using Lemma \ref{lemma: chernoff-hoeffeding inequality} and that $N^{0}_{T} \geq \ln T$, for a fixed $(i,x)$ we have
$$ \mathbb{P}\{ \widehat{p}_{i,x} \geq p_{i,x} - \frac{p}{2} \geq \frac{p}{2} \} \geq 1- e^{-\frac{p^2}{2}\ln T} = 1- \frac{1}{T^{\frac{p^2}{2}}}~~.$$
3. Recall that the effective number of arm pulls of arm $a_{i,x}$ at the end of $T$ rounds is  
$$E^{i,x}_T = N^{i,x}_T + \sum_{t=1}^T \mathbb{1}\{a_t =a_0 ~~\text{and}~~ X_i=x\}~~.$$ 
Hence, $E^{i,x}_T = N^{i,x}_T + \widehat{p}_{i,x}N^{0}_T$, where $\widehat{p}_{i,x}$ is as defined in part two of this lemma. Hence for any $i,x$ at the end of $T$ rounds if $\widehat{p}_{i,x} \geq \frac{p}{2}$ then $E^{i,x}_T \geq \frac{p N^0_T}{2}$. Further, as $N^0_T \geq \ln T$, it follows that at the end of $T$ rounds if $\widehat{p}_{i,x} \geq \frac{p}{2}$ then $E^{i,x}_T \geq \frac{p \ln T}{2}$. Hence, from the definition of $\widehat{\mu}_{i,x}(T)$ and Lemma \ref{lemma: chernoff-hoeffeding inequality}, at the end of $T$ rounds we have for any fixed $i,x$:
\begin{equation}\label{equation: theorem 3 bounding mu_i,x/gamma}
\mathbb{P}\Big\{ \Big|\frac{\widehat{\mu}_{i,x}(T)}{\gamma} - \frac{\mu_{i,x}}{\gamma} \Big| \geq \frac{d_0}{4} \Big| \widehat{p}_{i,x} \geq \frac{p}{2} \Big\} \leq 2e^{-\frac{\gamma^2 d_0^2}{16}p\ln T} = \frac{2}{T^{\frac{p\gamma^2d_0^2}{16}}}~.
\end{equation}
Finally by law of total probability,
\begingroup
\allowdisplaybreaks
\begin{align*}
    \mathbb{P}\Big\{ \Big|\frac{\widehat{\mu}_{i,x}(T)}{\gamma} - \frac{\mu_{i,x}}{\gamma} \Big| \geq \frac{d_0}{4} \Big\} &= \mathbb{P}\Big\{ \Big|\frac{\widehat{\mu}_{i,x}(T)}{\gamma} - \frac{\mu_{i,x}}{\gamma} \Big| \geq \frac{d_0}{4} \Big| \widehat{p}_{i,x} \geq \frac{p}{2} \Big\} \mathbb{P}\{\widehat{p}_{i,x} \geq \frac{p}{2}\} \\
    & ~~~~+\mathbb{P}\Big\{ \Big|\frac{\widehat{\mu}_{i,x}(T)}{\gamma} - \frac{\mu_{i,x}}{\gamma} \Big| \geq \frac{d_0}{4} \Big| \widehat{p}_{i,x} \leq \frac{p}{2} \Big\} \mathbb{P}\{\widehat{p}_{i,x} \leq \frac{p}{2}\} \\
      &\leq \mathbb{P}\Big\{ \Big|\frac{\widehat{\mu}_{i,x}(T)}{\gamma} - \frac{\mu_{i,x}}{\gamma} \Big| \geq \frac{d_0}{4} \Big| \widehat{p}_{i,x} \geq \frac{p}{2} \Big\} +  \mathbb{P}\{\widehat{p}_{i,x} \leq \frac{p}{2}\} \\
      &\leq \frac{2}{T^{\frac{p\gamma^2 d_0^2}{16}}} + \frac{1}{T^{\frac{p^2}{2}}}~.
\end{align*}
\endgroup
The last line in the above inequality follows from Equation \ref{equation: theorem 3 bounding mu_i,x/gamma} and part 2 of this lemma.
\end{proof}

\begin{lemma}\label{lemma: bounding beta}
Let $L = \arg\min_{t\in \N}\{ \frac{t^{\frac{p^2d_0^2}{16}}}{\ln t} \geq 15M\}$, and suppose \CRMPB pulls arms for $T$ rounds, where $T\geq \max(L, e^{\frac{50}{d_0^2}})$, and let $a^* \neq a_0$. Then at the end of $T$ rounds~ $\frac{8}{9d_0^2} \leq E[\beta^2] \leq \frac{50}{d_0^2}$. (Note that $\max(L, e^{\frac{50}{d_0^2}})$ is a finite constant dependent on instance constants $p,d_0$, and $M$.)
\end{lemma}
\begin{proof}
Recall that $\beta$ is set as in steps 11-14 in \CRMPB. We begin by making the following easy to see observations.
\begin{observation}\label{observation: inequality for beta}
\begin{enumerate}
    \item If $a^{*} \neq a_0$ then $d_0 = \frac{\mu_{a^*}}{\gamma} - \mu_0$. 
    \item Let $\widehat{\mu}^* = \max_{i,x}(\widehat{\mu}_{i,x}(T))$ (as computed in step 11 of \CRMPB). If $|\widehat{\mu}_0(T) - \mu_0| \leq  \frac{d_0}{4}$ and $|\frac{\widehat{\mu}_{i,x}(T)}{\gamma}   - \frac{\mu_{i,x}}{\gamma}| \leq  \frac{d_0}{4}$ for all $(i,x)$ then $\frac{d_0}{2} \leq \frac{\widehat{\mu}^*}{\gamma} - \widehat{\mu}_0(T) \leq \frac{3d_0}{2}$, and $\frac{32}{9d_0^2} \leq \beta^2 \leq \frac{32}{d_0^2}$. Notice that since $T\geq e^{\frac{50}{d_0^2}}$, $\frac{32}{d_0^2} \leq \ln T$.
\end{enumerate}
\end{observation}
Let $U_0$ be the event that $|\widehat{\mu}_0 - \mu_0| \leq \frac{d_0}{4}$, and for any $i,x$ let $U_{i,x}$ be the event $|\frac{\widehat{\mu}_{i,x}}{\gamma} - \frac{\mu_{i,x}}{\gamma}| \leq \frac{d_0}{4}$. Also let $U = (\cap_{i,x} U_{i,x}) \cap U_0$, and let $\overline{U}_0$, $\overline{U}_{i,x}$, and $\overline{U}$ denote the compliment of the events $U_0, U_{i,x}$, and $\overline{U}$ respectively.  From parts 1 and 3 of Lemma \ref{lemma: theorem 3 bounding mu's}, we have 
$$\mathbb{P} \Big\{\overline{U}_0 \Big\} \leq  \frac{2}{T^{\frac{d_0^2\ln T}{8}}}~, \text{~and}$$
$$\text{for a fixed } (i,x)~~~\mathbb{P}\Big\{ \overline{U}_{i,x} \Big\} \leq \frac{2}{T^{\frac{p\gamma^2 d_0^2}{16}}} + \frac{1}{T^{\frac{p^2}{2}}}~.$$
Hence applying union bound, 
\begin{align*}\mathbb{P}\{\overline{U}\} &\leq 4M\left(\frac{1}{T^{\frac{p\gamma^2d_0^2}{16}}} +  \frac{1}{T^{\frac{p^2}{2}}}\right) + \frac{2}{T^{\frac{d_0^2}{8}}} \\
&\leq 4M\left(\frac{1}{T^{\frac{p^2d_0^2}{16}}} +  \frac{1}{T^{\frac{p^2d_0^2}{16}}}\right) + \frac{2M}{T^{\frac{p^2d_0^2}{16}}} ~~~~~~~~~~~~\text{as }~\gamma \geq 1, p\leq 1, d_0\leq 1 \\
&\leq \frac{10M}{T^{\frac{p^2d_0^2}{16}}} =\delta~.
\end{align*}
We will use the above arguments to first show that $E[\beta^2] \geq \frac{8}{d_0^2}$. From part 2 of Observation \ref{observation: inequality for beta} we have that the event $U$ implies $\beta^{2} \geq \frac{32}{9d_0^2}$. Since $\mathbb{P}\{U\} \geq 1-\delta$,
$$E[\beta^{2}] \geq \frac{32}{9d_0^2} (1-\delta) = \frac{32}{9d_0^2} - \frac{32\delta}{9d_0^2} $$
Since $T$ satisfies $\frac{T^{\frac{p^2d_0^2}{16}}}{\ln T} \geq 15M$, this implies $\frac{32\delta}{9d_0^2} \leq \frac{24}{9d_0^2}$, and hence $E[\beta^2] \geq \frac{8}{9d_0^2}$. Similarly, from part 2 of Observation \ref{observation: inequality for beta} we have that the event $U$ implies $\beta^{2} \leq \frac{32}{d_0^2}$. Here, we use that if $U$ does not hold then $\beta^2 \leq \ln T$. Hence
$$E[\beta^{2}] \leq \frac{32}{d_0^2} (1-\delta) + \delta\ln T \leq \frac{32}{d_0^2} + \delta \ln T ~.$$
Since $T$ satisfies $\frac{T^{\frac{p^2d_0^2}{16}}}{\ln T} \geq 15M$, we have $\delta \ln T \leq \frac{18}{d_0^2}$, and hence 
$ E[\beta^{2}] \leq \frac{50}{d_0^2}$.
\end{proof}
\begin{lemma}\label{lemma: bounding the number of suboptimal pull for given T}
Suppose the algorithm pulls the arms for $T$ rounds and if $a^* \neq a_{i,x}$. Then
$$E[N^{i,x}_{T}|T] \leq \max\left(0, \frac{8\ln T}{d_{i,x}^2} + 1 - p_{i,x}E[N^0_{t}]\right) + \frac{\pi^2}{3}~.$$
Further if $a^* \neq a_0$ then
$$E[N_{0,t}|T] \leq \max\Big( E[\beta^2]\ln T, ~\frac{8\ln T}{d_{0}^2} + 1\Big) + \frac{\pi^2}{3}, ~.$$
\end{lemma}
\begin{proof}
For ease of notation we denote $E[N^{i,x}_{T}|T]$ as $E[N^{i,x}_{T}]$. Observe that
\begin{equation}\label{equation: value of N^i,x_T}
    N^{i,x}_{T} = \sum_{t\in T} \mathbb{1}\{a(t) = a_{i,x}\} ~.
\end{equation}
Since $E^{i,x}_T = N^{i,x}_T + \sum_{t\in [T]}\mathbb{1}\{a(t) =a_0~~\text{and}~~ X_i=x\}$, if $E^{i,x}_T = \ell$ then $N^{i,x}_T = \max(0,\ell - \sum_{t\in [T]}\mathbb{1}\{a(t) =a_0~~\text{and}~~ X_i=x\})$. We use this to rewrite Equation \ref{equation: value of N^i,x_T} as follows
\begin{equation}\label{equation: value of N^i,x_T after using E^i,x_T}
N^{i,x}_{T} \leq \text{max}(0,\ell - \sum_{t\in [T]}\mathbb{1}\{a(t) =a_0~~\text{and}~~ X_i=x\}) + \sum_{t\in T}\mathbb{1}\{a(t) = a_{i,x}, E^{i,x}_{t} \geq \ell\}~. 
\end{equation}
We require the following observation which is easy to prove.
\begin{observation}\label{observation: lemma 3}
$\sum_{t\in [T]}E[\mathbb{1}\{a(t) =a_0~~\text{and}~~ X_i=x\}] = p_{i,x}E[N^0_T]$~.
\end{observation}
\begin{proof}
Observe that 
$$E[\sum_{t\in [T]}\mathbb{1}\{a(t) =a_0~~\text{and}~~ X_i=x\}] = \sum_{t\in [T]} E[\mathbb{1}\{a(t) =a_0~~\text{and}~~ X_i=x\}] = \sum_{t\in [T]} \mathbb{P}\{\mathbb{1}\{a(t) =a_0~~\text{and}~~ X_i=x\}\}$$
Also observe that
$$\mathbb{P}\{\mathbb{1}\{a(t) =a_0~~\text{and}~~ X_i=x\} = \mathbb{P}\{\mathbb{1}\{X_i=x\} \mid a(t) =a_0\}\}\cdot \mathbb{P}\{a(t)= a_0\} = p_{i,x} \mathbb{P}\{a(t)= a_0\}~.$$
\end{proof}
We continue by taking expectation on both sides of Equation \ref{equation: value of N^i,x_T after using E^i,x_T} and use Observation \ref{observation: lemma 3},
\begin{equation}\label{equation: bounding the expected number of pulls of sub-optimal arm}
    E[N^{i,x}_{T}] \leq \text{max}\left(0, \ell - p_{i,x}E[N^0_{t}]\right) + \sum_{t\in [\ell+1, T]} \mathbb{P}\{a(t) = a_{i,x}, E^{i,x}_{t} \geq \ell\} ~.
\end{equation}
Now we bound $\sum_{t\in [\ell+1,T]} \mathbb{P}\{a(t) = a_{i,x}, E^{i,x}_{t} \geq \ell\}$, and assuming $a^*\neq a_0$. The proof for $a^*=a_0$ is similar. Before proceeding we make a note of few notations. We use $E^{a^*}_T$ to denote the effective number of pulls of $a^*$ at the end of $T$ rounds. Also, for better clarity in the arguments below, we use $\widehat{\mu}_{i,x}(E^{i,x}_T,T)$ (instead of $\widehat{\mu}_{i,x}(T)$) and $\widehat{\mu}_{0}(N^{0}_T,T)$ (instead of $\widehat{\mu}_{0}(T)$) to denote the empirical estimates of $\mu_{i,x}$ and $\mu_0$  computed by \CRMPB at the end of $T$ rounds using $E^{i,x}_T$ and $N^0_T$ samples respectively. Let $C = \sum_{t\in [\ell+1,T]} \mathbb{P}\{a(t) = a_{i,x}, E^{i,x}_{t} \geq \ell\}$ for convenience. Then
\begin{align*}
    C &= \sum_{t\in [\ell,T-1]} \mathbb{P}\Bigg\{\frac{\widehat{\mu}_{a^*}(E^{a^*}_t,t)}{\gamma} + \sqrt{\frac{2\ln t}{\gamma^2 E^{a^*}_{t}}} \leq  \frac{\widehat{\mu}_{i,x}(E^{i,x}_t, t)}{\gamma} + \sqrt{\frac{2\ln (t)}{\gamma^2 E^{i,x}_{t}}},~~ E^{i,x}_{t} \geq \ell \Bigg\} \\
    &\leq  \sum_{t\in [0,T-1]} \mathbb{P}\Bigg\{ \text{min}_{s\in [0,t]}\frac{\widehat{\mu}_{a^*}(s,t)}{\gamma} + \sqrt{\frac{2\ln t}{\gamma^2 s}} \leq \text{max}_{s_j\in [\ell-1,t]} \frac{\widehat{\mu}_{i,x}(s_j,t)}{\gamma} + \sqrt{\frac{2\ln t}{\gamma^2 s_j}} \Bigg\} \\
    &\leq  \sum_{t\in T} \sum_{s\in [0,t-1]} \sum_{s_j \in [\ell-1,t]} \mathbb{P}\Bigg\{\frac{\widehat{\mu}_{a^*}(s,t)}{\gamma} + \sqrt{\frac{2\ln t}{\gamma^2 s}} \leq \frac{\widehat{\mu}_{i,x}(s_j, t)}{\gamma} + \sqrt{\frac{2\ln t}{\gamma^2 s_j}}\Bigg\}
\end{align*}
If $\frac{\widehat{\mu}_{a^*}(s,t)}{\gamma} + \sqrt{\frac{2\ln t}{\gamma^2 s}} \leq \frac{\widehat{\mu}_{i,x}(s_j,t)}{\gamma} + \sqrt{\frac{2\ln t}{\gamma^2 s_j}}$ is true then at least one of the following events is true
\begin{subequations}
\begin{align}
   \frac{\widehat{\mu}_{a^*}(s,t)}{\gamma} &\leq \frac{\mu_{a^*}}{\gamma} - \sqrt{\frac{2\ln t}{\gamma^2 s}} ~, \label{equation: ucb event a}\\
   \frac{\widehat{\mu}_{i,x}(s_j,t)}{\gamma} &\geq \frac{\mu_{i,x}}{\gamma} + \sqrt{\frac{2\ln t}{\gamma^2 s_j}} ~, \label{equation: ucb event b}\\
   \frac{\mu_{a^*}}{\gamma} &\leq  \frac{\mu_{i,x}}{\gamma} + 2\sqrt{\frac{2\ln t}{\gamma^2 s_j}}~. \label{equation: ucb event c}
\end{align}
\end{subequations}
The probability of the events in Equations \ref{equation: ucb event a} and \ref{equation: ucb event b} can be bounded using Lemma \ref{lemma: chernoff-hoeffeding inequality},
$$\mathbb{P}\Bigg\{\frac{\widehat{\mu}_{a^*}(s,t)}{\gamma} \leq \frac{\mu_{a^*}}{\gamma} - \sqrt{\frac{2\ln t}{\gamma^2 s}}\Bigg\} \leq t^{-4} ~,$$ 
$$\mathbb{P}\Bigg\{\frac{\widehat{\mu}_{i,x}(s_j,t)}{\gamma} \geq \frac{\mu_{i,x}}{\gamma} + \sqrt{\frac{2\ln t}{\gamma^2 s_j}}\Bigg\} \leq t^{-4} ~.$$ 
Also if $\ell \geq \lceil \frac{8\ln T}{d_{i,x}^2} \rceil$  then the event in Equation \ref{equation: ucb event c} is false, i.e. $\frac{\mu_{a^*}}{\gamma} >  \frac{\mu_{i,x}}{\gamma} + 2\sqrt{\frac{2\ln t}{\gamma^2 s_j}}$ (as $\gamma \geq 1$). Thus we set $\ell = \frac{8\ln T}{d_{i,x}^2} + 1 \geq \lceil\frac{8\ln T}{d_{i,x}^2}\rceil$, which implies
\begin{equation}\label{equation: probability of a suboptimal pull}
\sum_{t\in T} \mathbb{P}\{a(t) = a_{i,x}, E^{i,x}_{t} \geq \ell\} \leq \sum_{t\in [T]}\sum_{s\in [T]}\sum_{s_j\in [\ell, T]} 2t^{-4} \leq \frac{\pi^2}{3}
\end{equation}
If $a^* = a_0$ then using the exact arguments as above we can show that Equation \ref{equation: probability of a suboptimal pull} still holds. Hence, using Equations \ref{equation: bounding the expected number of pulls of sub-optimal arm} and \ref{equation: probability of a suboptimal pull} we have if $a^* \neq a_{i,x}$ then
$$E[N^{i,x}_{T}] \leq \text{max}\left(0, \frac{8\ln T}{d_{i,x}^2} + 1 - p_{i,x}E[N^{0}_{t}]\right) + \frac{\pi^2}{3} ~.$$
The arguments used to bound $E[N^0_T|T]$ (denoted $E[N^0_T]$ for convenience), when $a^*\neq a_0$ is similar. In this case the equation corresponding to Equation \ref{equation: bounding the expected number of pulls of sub-optimal arm} is 
\begin{equation}\label{equation: bounding the expected number of pulls of sub-optimal arm a_0}
    E[N^{0}_{T}] \leq \text{max}\Big(E[\beta^2]\ln T, ~\ell\Big) + \sum_{t\in [\ell+1, T]} \mathbb{P}\{a(t) = a_0, N^{0}_{t} \geq \ell\} ~.
\end{equation}
Also the same arguments as above can be used to show that for $\ell =  \frac{8\ln T}{d_{0}^2} + 1$,
\begin{equation}\label{equation: probability of a suboptimal pull a_0}
\sum_{t\in T} \mathbb{P}\{a(t) = a_{0}, N^{0}_{t} \geq \ell\}  \leq \frac{\pi^2}{3}~.
\end{equation}
Finally using Equations \ref{equation: bounding the expected number of pulls of sub-optimal arm a_0} and \ref{equation: probability of a suboptimal pull a_0}, we have
$$E[N^{0}_{T}] \leq \text{max}\left(E[\beta^{2}]\ln T, ~\frac{8\ln T}{d_{0}^2} + 1 \right)  + \frac{\pi^2}{3}~.$$
\end{proof}

\begin{lemma}\label{lemma: number of pulls of the arm a_0 when it is optimal}
If $a^* = a_0$ and suppose the algorithm pulls the arms for $T$ rounds then 
$$E[N^0_T|T] \geq T - \left( 2M(1+\frac{\pi^2}{3}) \sum_{i,x} \frac{8\ln T}{d_{i,x}^2} \right)~.$$
\end{lemma}
\begin{proof}
For convenience, we denote $E[N^{i,x}_T|T]$ and $E[N^0_T|T]$ as $E[N^{i,x}_T]$ and $E[N^0_T]$ respectively. At the end of $T$ rounds we have
$$N^{0}_T + \sum_{i,x} N^{i,x}_T = T~. $$
Taking expectation on both sides of the above equation and rearranging the terms we have,
$$E[N^{0}_T] = T -  \sum_{i,x} E[N^{i,x}_T]~. $$
Now we use Lemma \ref{lemma: bounding the number of suboptimal pull for given T} to conclude that
$$E[N^{0}_T] \geq T - \left( 2M(1+\frac{\pi^2}{3}) \sum_{i,x} \frac{8\log T}{d_{i,x}^2} \right)~. $$
\end{proof}
Before we bound the regret of the algorithm we make the following observation regarding $T$, which is the number of rounds \CRMPB pulls the arms before exhausting the budget $B$:
\begin{equation}\label{equation: bounding T}
   \frac{B}{\gamma}\leq  T \leq B  ~~\Rightarrow ~~ \frac{B}{\gamma} \leq E_T[T] \leq B  ~.
\end{equation}
Now are ready to bound the expected cumulative regret of $\CRMPB$ for the two cases:

\noindent \textbf{Case a} ($a^* = a_0$): In this case we bound the expected cumulative regret of \CRMPB for $B$ satisfying
\begin{equation}\label{equation: constraint on B}
    \frac{B}{\gamma} \geq \frac{1}{p_{i,x}}(1+\frac{8\ln B}{d_{i,x}^2}) +  \left( 2M(1+\frac{\pi^2}{3}) \sum_{i,x} \frac{8\ln B}{d_{i,x}^2}\right)~.
\end{equation}
Observe that the constraint on $B$ in Equation \ref{equation: constraint on B} is satisfied for any large $B$. We begin by making the following observation which shows that in this case the expected number of pulls of a sub-optimal arm is bounded by a constant for any large $B$. Observe that the constraint on $B$ in Observation \ref{observation: bounding the sub optimal pull when a_0 is optimal for large budget} is satisfied for any large $B$. 
\begin{observation}\label{observation: bounding the sub optimal pull when a_0 is optimal for large budget}
Let $a^* = a_0$, and $T$ be the number of rounds \CRMPB pulls the arms before the budget $B$ is exhausted, where $B$ satisfies the constraint in Equation \ref{equation: constraint on B}. Then $E_T[N^{i,x}_T] \leq \frac{\pi^2}{3}$.
\end{observation}
\begin{proof}
From Lemmas \ref{lemma: bounding the number of suboptimal pull for given T} and \ref{lemma: number of pulls of the arm a_0 when it is optimal} for any $T$ satisfying 
\begin{equation}\label{equation: constraint on T}
    T \geq \frac{1}{p_{i,x}}(1+\frac{8\ln T}{d_{i,x}^2}) +  \left( 2M(1+\frac{\pi^2}{3}) \sum_{i,x} \frac{8\ln T}{d_{i,x}^2}\right)
\end{equation}
we have $E[N^{i,x}_T| T] \leq \frac{\pi^2}{3}$. Notice that the constraint on $T$ in Equation \ref{equation: constraint on T} is the same as the constraint on $\frac{B}{\gamma}$ in Equation \ref{equation: constraint on B}. Moreover, observe that if $\frac{B}{\gamma}$ satisfies the constraint in Equation \ref{equation: constraint on B} then $T \geq \frac{B}{\gamma}$ satisfies Equation \ref{equation: constraint on T} with probability $1$. Hence, $E_T[N^{i,x}_T|T] \leq \frac{\pi^2}{3}$.
\end{proof}
Next observe that in this case $G_{B}$ (see Equation \ref{eq:cumulative regret}) is $B\mu_{0}$, i.e the optimal solution is to play arm $a_0$ in all the rounds. We require the following observation which lower bounds $E_T[T]$ in terms of $B$, which is the total number of rounds played by the optimal solution.
\begin{observation}\label{observation: lower bound expected number of pulls}
Let $a^* = a_0$, and $T$ be the number of rounds \CRMPB pulls the arms before the budget $B$ is exhausted, where $B$ satisfies the constraint in Equation \ref{equation: constraint on B}. Then 
$E_T[T] \geq B-1 - \frac{2M\pi^2(\gamma -1)}{3}$.
\end{observation}
\begin{proof}
Let $c_{a_t}$ denote the cost of arm $a_t$ pulled at time $t \leq T$. That is $c_{a_t} =\gamma$ if $a_t = a_{i,x}$ and $c_{a_t} = 1$ if $a_t = a_0$. Then the following is always true, as \CRMPB pulls arms till the budget is the exhausted:
\begin{equation}
    B-1 \leq \sum_{t\in [T]} c_{a_t}~.
\end{equation}
Taking expectation over $T$ and the sequence of arm pulls $\{a_t\}$ made by \CRMPB, on both sides of the above equation, we have
\begingroup
\allowdisplaybreaks
\begin{align*}
    B-1 &\leq E_{T,\{a_t\}}\Big[\sum_{t\in [T]} c_{a_t} \Big] \\
    &\leq E_{T}\Big[ E_{\{a_t\}}[\sum_{t\in [T]} c_{a_t}] \Big] \\
    &\leq E_{T}\Big[ \sum_{t\in [T]} \Big(\mathbb{P}\{a_t = a_0\} + \gamma(\sum_{i,x} \mathbb{P}\{a_t = a_{i,x}\} \Big) \Big] \\
     &\leq E_{T}\Big[ T + \sum_{t\in [T]}(\gamma-1)(\sum_{i,x} \mathbb{P}\{a_t = a_{i,x}\} ) \Big]\\
     &\leq E_{T}[T] + E_{T}\Big[\sum_{i,x}(\gamma-1)( \sum_{t\in [T]} \mathbb{P}\{a_t = a_{i,x}\} ) \Big]\\
      &\leq E_{T}[T] + E_{T}\Big[\sum_{i,x}(\gamma-1)E[N^{i,x}_T|T] \Big]~. \\
\end{align*}
\endgroup
The third line in the above set of equations follows by using $\mathbb{P}\{a_t = a_0\} = 1 - \sum_{i,x}\mathbb{P}\{a_t = a_{i,x}\}$. Finally from Observation \ref{observation: bounding the sub optimal pull when a_0 is optimal for large budget}, we have $E[N^{i,x}_T] \leq \frac{\pi^2}{3}$. Substituting this in the last line of the above equation, we have $E_T[T] \geq B-1 - \frac{2M\pi^2(\gamma -1)}{3}$.
\end{proof}
Finally we bound the expected cumulative regret of \CRMPB when $a^* = a_0$ as follows: 
\begingroup
\allowdisplaybreaks
\begin{align*}
   E[R_{\CRMPB}(B)] ~\leq & G_B - E_{T,\{a_t\}}\left[ \sum_{t\in [T]} \mu_{a_t} \right] \\
    ~ \leq & B\mu_0 - E_{T}\left[\sum_{t=1}^T E_{\{a_t\}}[\mu_{a_t}]  \right]\\
    ~\leq & E_T\left[ B\mu_0 - \sum_{t=1}^T E_{\{a_t\}}[\mu_{a_t}]  \right] \\
    ~\leq & E_T\left[ B\mu_0 - \sum_{t=1}^T \sum_{a\in \mathcal{A}} \mu_{a} \mathbb{P}\{a_t =a\}  \right] \\
    ~\leq & E_T\left[ (B - T)\mu_{0} + \sum_{t=1}^T (\mu_{0} - \sum_{a\in \mathcal{A}} \mu_{a} \mathbb{P}\{a_t =a\} ) \right] \\
    ~\leq & E_T\left[ (B - T)\mu_{0}\right] + E_T\left[\sum_{t=1}^T \sum_{\Delta_a>0} \Delta_a \mathbb{P}\{a_t =a\} ) \right] ~.
\end{align*}
\endgroup
Thus, from Observations \ref{observation: bounding the sub optimal pull when a_0 is optimal for large budget} and \ref{observation: lower bound expected number of pulls}, we have 
$$E[R_{\CRMPB}(B)] \leq 1 + \frac{2M\pi^2(\gamma -1)}{3} + \sum_{\Delta_a>0} \Delta_a \frac{\pi^2}{3} ~.$$
Observe that the expected  regret of \CRMPB is bounded by a constant for large $B$ and hence $O(1)$.

\noindent \textbf{Case b} ($a^* \neq a_0$): In this case we bound the expected cumulative regret of \CRMPB for $B$ satisfying $B \geq \max(L, e^{\frac{50}{d_0^2}})$, where $L$ is as in Lemma \ref{lemma: bounding beta}. Observe that the constraint is satisfied for any large $B$. Let $T$ be the number of rounds \CRMPB pulls the arms before exhausting the budget $B$. Then from Equation \ref{equation: constraint on T}, we have $T\geq \max(L, e^{\frac{50}{d_0^2}})$. Hence, from Lemmas \ref{lemma: bounding beta} and \ref{lemma: bounding the number of suboptimal pull for given T}, and as $T\leq B$ (from Equation \ref{equation: bounding T}), we have for $a^{*} \neq a_{i,x}$
\begin{equation}\label{equation: bounding sub-optimal pulls when a^* not a_0}
E_T\left[E[N^{i,x}_{T}|T] \right]  \leq  \max\left(0, 1+ 8\ln B\left(\frac{1}{d_{i,x}^2} - \frac{p_{i,x}}{9d_0^2}\right) \right) + \frac{\pi^2}{3}~,
\end{equation}
\begin{equation}\label{equation: bounding pulls of a_0 when a^* not a_0}
\text{and}~~~~E_T\left[E[N^0_{T}|T]\right] \leq \frac{50\ln B}{d_{0}^2} + \frac{\pi^2}{3}~~.
\end{equation}
Also observe that in this case $G_{B}$ is at most $\frac{B\mu_{a^*}}{\gamma}$. Below we bound the expected cumulative regret of \CRMPB when $a^*\neq a_{0}$
\begingroup
\allowdisplaybreaks
\begin{align*}
   E[R_{\CRMPB}(B)] ~\leq & \frac{B\mu_{a^*}}{\gamma} - E_{T,\{a_t\}}\left[ \sum_{t\in [T]} \mu_{a_t} \right] \\
    ~ \leq &~ \frac{B\mu_{a^*}}{\gamma} - E_{T}\left[\sum_{t=1}^T E_{\{a_t\}}[\mu_{a_t}]  \right]\\
    ~ \leq &~ E_T\left[ \frac{B\mu_{a^*}}{\gamma} - \sum_{t=1}^T E_{\{a_t\}}[\mu_{a_t}]  \right] \\
    ~ \leq &~ E_T\left[ \frac{B\mu_{a^*}}{\gamma} - \sum_{t=1}^T \sum_{a\in \mathcal{A}} \mu_{a} \mathbb{P}\{a_t =a|T\}  \right] \\
    ~ \leq &~ E_T\left[ \Big(\frac{B}{\gamma} - T\Big)\mu_{a^*} + \sum_{t=1}^T (\mu_{a^*} - \sum_{a\in \mathcal{A}} \mu_{a} \mathbb{P}\{a_t =a|T\} ) \right] \\
    ~ \leq &~ E_T\left[ \Big(\frac{B}{\gamma} - T\Big)\mu_{a^*}\right] + E_T\left[\sum_{t=1}^T \sum_{\Delta_a>0} \Delta_a \mathbb{P}\{a_t =a|T\} ) \right]~. \\
\end{align*}
\endgroup
Now observe that as $T\geq \frac{B}{\gamma}$, $E_T[(\frac{B}{\gamma} - T)\mu_{a^*}] \leq 0$. Also note that $E_T[\sum_{t=1}^T \mathbb{P}\{a_t =a|T\}] = E_T[N^a_T|T]$. Using this and Equations \ref{equation: bounding sub-optimal pulls when a^* not a_0} and \ref{equation: bounding pulls of a_0 when a^* not a_0}, we have our result as follows:
\begin{align*}
  E[R_{\CRMPB}(B)]  ~ \leq &~ \Delta_0 E_T\left[E[N^0_{T}|T]\right] + \sum_{\Delta_{i,x}>0} \Delta_{i,x} E_T\left[E[N^{i,x}_{T}|T]\right]\\
    ~ \leq &~ \Delta_0\Big(\frac{50\ln B}{d_{0}^2} + \frac{\pi^2}{3}\Big) + \sum_{\Delta_{i,x}>0} \Delta_{i,x} \max\left(0, 1+ 8\ln B\left(\frac{1}{d_{i,x}^2} - \frac{p_{i,x}}{9d_0^2}\right) \right) + \frac{\pi^2}{3}~.
\end{align*}

Hence, we have that the expected cumulative regret of \CRMPB is:
$$
E[R_{\CRMPB}(B)] \leq
\begin{cases}
1 + \frac{2M\pi^2(\gamma -1)}{3} + \sum_{\Delta_a>0} \Delta_a \frac{\pi^2}{3} \hspace{7.1cm}\text{when~} a^* = a_0\\
\Delta_0\Big(\frac{50\ln B}{d_{0}^2} + \frac{\pi^2}{3}\Big) + \sum_{\Delta_{i,x}>0} \Delta_{i,x} \max\left(0, 1+ 8\ln B\left(\frac{1}{d_{i,x}^2} - \frac{p_{i,x}}{9d_0^2}\right) \right) + \frac{\pi^2}{3} \hspace{0.5cm}\text{when~} a^* \neq a_0
\end{cases}
$$

\subsection{Proof of Theorem \ref{theorem: cumulative regret for general graphs}}\label{secappendix: proof of cum regret for general graphs}
%
Throughout this proof $a$ and $\mathbf{y}$ indexes the sets $\mathcal{A}$ and $S^n$ respectively. Let $\delta$, $L_1$, $L_{2,a}$ and $L_a$ for all $a$, be as in the theorem statement. Let $a^{*} = \arg\max_a (\mu_a)$. As is standard in MAB literature, we assume without loss of generality that $a^*$ is unique. Further, let $\Delta_a = \mu_{a^*} - \mu_a$. 
The regret upper bound is proved using Lemmas \ref{lemma: theorem 4 1} and \ref{lemma: the bound on the number of sub-optimal interventions}.

\begin{lemma}\label{lemma: theorem 4 1}
\label{lemma: hoeffding adaptation}
Let $T$ be the number of rounds \CUCB2 has pulled the arms. Then for $T\geq L_1$ the following holds:
\begin{enumerate}
    \item For all $\mathbf{y}$ such that $c_{\mathbf{y}} >0$, $\mathbb{P}\big\{N_{\mathbf{y},T} \leq \frac{E[N_{\mathbf{y},T}]}{2}\big\} \leq e^{-\frac{E[N_{\mathbf{y},T}]^2}{2T}}$~,
    \item For all $\mathbf{y}$ such that $c_{\mathbf{y}} >0$ and for any $\varepsilon_{\mathbf{y}} \geq 0$, $\mathbb{P}\{|\widehat{\mu}_{\mathbf{y}}(T)-\mu_{\mathbf{y}}| \geq \varepsilon_{\mathbf{y}}\} \leq 2e^{-c_{\mathbf{y}}^2T\varepsilon_{\mathbf{y}}^2} + e^{-\frac{c_{\mathbf{y}}^2T}{2}}$~,
    \item For all $a$, $\mathbb{P}\Big(|\widehat{\mu}_{a}(T) - \mu_a| \geq \sqrt{\frac{\log (k^n T^2/2)}{T}}\zeta_a\Big) \leq \frac{2}{T^2}$~.
\end{enumerate}
\end{lemma}
\begin{proof}
1. Part 1 of the lemma follows from Lemma \ref{lemma: chernoff-hoeffeding inequality}.

2. Using Lemma \ref{lemma: chernoff-hoeffeding inequality} again, it follows that for all ${\mathbf{y}}$ such that $c_{\mathbf{y}}>0$, and for all $\varepsilon_{\mathbf{y}} \geq 0$,
\begin{equation}\label{equation: bored}
 \mathbb{P}\Big\{|\widehat{\mu}_{\mathbf{y}}(T) - \mu_{\mathbf{y}} \Big| \geq \varepsilon_{\mathbf{y}} \mid N_{\mathbf{y},T} > \frac{E[N_{\mathbf{y},T}]}{2}\Big\} \leq 2e^{-E[N_{\mathbf{y},T}]\varepsilon_{\mathbf{y}}^2}~.    
\end{equation}
Hence, for all ${\mathbf{y}}$ such that $c_{\mathbf{y}}>0$, using the law of total probability we have
\begin{align*}
 \mathbb{P}(|\widehat{\mu}_{\mathbf{y}}(T) - \mu_{\mathbf{y}}| \geq \varepsilon_{\mathbf{y}}) &= \mathbb{P}\Big \{|\widehat{\mu}_{\mathbf{y}}(T) - \mu_{\mathbf{y}}| \geq \varepsilon_{\mathbf{y}} \Big| N_{\mathbf{y},T} > \frac{E[N_{\mathbf{y},T}]}{2}\Big\} \mathbb{P}\Big\{N_{\mathbf{y},T} > \frac{E[N_{\mathbf{y},T}]}{2}\Big\} + \\  &~~~~~~~ \mathbb{P}\Big\{|\widehat{\mu}_{\mathbf{y}}(T) - \mu_{\mathbf{y}}| \geq \varepsilon_{\mathbf{y}} \Big| N_{\mathbf{y},T} \leq \frac{E[N_{\mathbf{y},T}]}{2}\Big\} \mathbb{P}\Big\{N_{\mathbf{y},T} \leq \frac{E[N_{\mathbf{y},T}]}{2}\Big\}\\
  &\leq \mathbb{P}\Big\{|\widehat{\mu}_{\mathbf{y}}(T) - \mu_{\mathbf{y}}| \geq \varepsilon_{\mathbf{y}} \Big| N_{\mathbf{y},T} > \frac{E[N_{\mathbf{y},T}]}{2}\Big\} + \mathbb{P}\Big\{N_{\mathbf{y},T} \leq \frac{E[N_{\mathbf{y}},T]}{2}\Big\} \\
  &\leq 2e^{-E[N_{\mathbf{y},T}]\varepsilon_{\mathbf{y}}^2} + e^{-\frac{E[N_{\mathbf{y},T}]^2}{2T}}\\
  &\leq 2e^{-c_{\mathbf{y}}T\varepsilon_{\mathbf{y}}^2} + e^{-\frac{c_{\mathbf{y}}^2T}{2}}\\
  &\leq 2e^{-c_{\mathbf{y}}^2T\varepsilon_{\mathbf{y}}^2} + e^{-\frac{c_{\mathbf{y}}^2T}{2}}~.
\end{align*}
The second line in the above equations follows from Equation \ref{equation: bored} and part one of this lemma. The last two inequalities follow by observing that $E[N_{\mathbf{y},T}] \geq c_{\mathbf{y}} T \geq c_{\mathbf{y}}^2T$. This is true as for each $\mathbf{y}$, $c_{\mathbf{y}} =\text{min}_{a} \mathbb{P}\{Pa(Y) = \mathbf{y} \mid do(a)\}$, and hence $0< c_{\mathbf{y}} \leq 1$.

3. Let $\varepsilon_{\mathbf{y}} = \sqrt{\frac{\log (k^n T^2/2)}{c_{\mathbf{y}}^2T}}$ if $c_{\mathbf{y}}>0$, and $\varepsilon_{\mathbf{y}} = 0$ if $c_{\mathbf{y}} = 0$. Since the parent distributions have the same non-zero support, and as $\widehat{\mu}_{a}(T) = \sum_{\mathbf{y}} \widehat{\mu}_{\mathbf{y}}(T) \mathbb{P}\{Pa(Y)=\mathbf{y}|do(a)\}$, the event 
$$|\widehat{\mu}_{a}(T) - \mu_a| \geq \sum_{\mathbf{y},c_{\mathbf{y}}>0}\varepsilon_{\mathbf{y}} \mathbb{P}\{Pa(Y)=\mathbf{y}|do(a)\}$$ 
implies there is a $\mathbf{y}$ such that $c_{\mathbf{y}}>0$ and $\{|\widehat{\mu}_{\mathbf{y}}(T) - \mu_{\mathbf{y}}| \geq \varepsilon_{\mathbf{y}}\}$. Hence, using part 2 of this lemma and applying union bound over all $\mathbf{y}$ such that $c_{\mathbf{y}}>0$, we have for every $a$ 
$$ \mathbb{P}\Big\{|\widehat{\mu}_{a}(T) - \mu_a| \geq \sum_{\mathbf{y},c_{\mathbf{y}}>0}\varepsilon_{\mathbf{y}} \mathbb{P}\{Pa(Y)=\mathbf{y} | do(a)\}\Big\} \leq \sum_{\mathbf{y},c_{\mathbf{y}}>0} \big(2e^{-c_{\mathbf{y}}^2T\varepsilon_{\mathbf{y}}^2} + e^{-\frac{c_{\mathbf{y}}^2T}{2}}\big)~.$$
Substituting the values of $\varepsilon_{\mathbf{y}}$ and using $\zeta_a = \sum_{\mathbf{y}, c_{\mathbf{y}}>0} \frac{\mathbb{P}\{Pa(Y)= \mathbf{y}\mid do(a) \}}{c_{\mathbf{y}}}$ in the above equation, we have  
\begin{align*}
    \mathbb{P}\Big\{|\widehat{\mu}_{a}(T) - \mu_a| \geq \sqrt{\frac{\log (k^n T^2/2)}{T}}\zeta_a\Big\} & \leq \frac{1}{T^2} + \sum_{\mathbf{y}, c_{\mathbf{y}}>0} e^{-\frac{c_{\mathbf{y}}^2T}{2}}\\
    &\leq \frac{1}{T^2} + k^n e^{-\delta^2T/2}
\end{align*}
where $\delta = \min_{c_\mathbf{y}>0} c_{\mathbf{y}}$. Since $T\geq L_1$, $T \geq \frac{2\log(k^nT^2)}{\delta^2}$. This implies 
$k^n e^{-\delta^2T/2} \leq \frac{1}{T^2}$, and
$$ \mathbb{P}\Big\{|\widehat{\mu}_{a} - \mu_a| \geq \sqrt{\frac{\log (k^n T^2/2)}{T}}\zeta_a\Big\} \leq \frac{2}{T^2}~.$$
\end{proof}

\begin{lemma}\label{lemma: the bound on the number of sub-optimal interventions}
Let $a\in A$ be a sub-optimal intervention. Then the expected number of times intervention $a$ is made after $L_a = \max\{L_1,L_{2,a}\}$ rounds is at most $\frac{2\pi^2}{3}$.
\end{lemma}
\begin{proof}
For ease of notation, we denote $ \sqrt{\frac{\log (k^n t^2/2)}{t}}\zeta_a$ as $c_{a,t}$. Note that $c_{a,t}$ is the confidence radius of intervention $a$ \CUCBTwo maintains at the end of $t$ rounds. Further, let $N'_{a,T}$ denote the number of times the algorithm performs intervention $a$ from time $L_{a}+1$ to time $T\geq L_{a}$, and also let $a_t$ denote the intervention performed at time $t$. Hence,
\begin{equation}\label{equation: renaming the upper on sub-optimal pulls}
N'_{a,T}  = \sum_{t = L_{a}+1}^T \mathbb{1}\Big\{a_{t} = a\Big\} ~.
\end{equation}
Note that $a_{t} = a$ implies $\Bar{\mu}_{a^*}(t-1) \leq \Bar{\mu}_{a}(t-1)$~ i.e.~ $\widehat{\mu}_{a^*}(t-1) + c_{a^*,t-1} \leq \widehat{\mu}_{a}(t-1) + c_{a,t-1}$~. Hence from Equation \ref{equation: renaming the upper on sub-optimal pulls}, we have
$$   N'_{a,T}  \leq \sum_{t = L_a}^{T-1} \mathbb{1}\Big\{ \widehat{\mu}_{a^*}(t) + c_{a^*,t} \leq \widehat{\mu}_{a}(t) + c_{a,t} \Big\} ~.$$
The event $\widehat{\mu}_{a^*,t} + c_{a^*,t} \leq \widehat{\mu}_{a,t} + c_{a,t}$ implies that at least one of the following events is true
\begin{align}
    \big\{\widehat{\mu}_{a^*}(t) \leq \mu_{a^*}- c_{a^*,t}\big\} \label{eq:EventA} \\
    \big\{\widehat{\mu}_{a}(t) \geq \mu_a + c_{a,t}\big\} \label{eq:EventB}\\
    \big\{\mu_{a^*} < \mu_a + 2c_{a,t}\big\} \label{eq:EventC}
\end{align}
Since $t\geq L_a \geq L_1$, using Lemma \ref{lemma: hoeffding adaptation} the probability of the events in Equations \ref{eq:EventA} and \ref{eq:EventB} can be bounded as:
\begin{align*}
    \mathbb{P}\big\{\widehat{\mu}_{a^*}(t) \leq \mu_{a^*} - c_{1,t}\big\} \leq 2t^{-2}~,\\
    \mathbb{P}\big\{\widehat{\mu}_{a}(t) \geq \mu_a + c_{a,t}\big\} \leq 2t^{-2}~.\\
\end{align*}
The event in equation \ref{eq:EventC}~ $\big\{\mu_{a^*} < \mu_a + 2c_{a,t}\big\}$ can be written as $\Big\{\mu_{a^*} - \mu_a - 2\sqrt{\frac{\log (k^n t^2/2)}{t}}\zeta_a < 0 \Big\}$. Substituting $\Delta_a = \mu_{a^*} - \mu_a$ and since $t \geq L_a\geq L_{2,a}$, we have
\begin{align}
    \mathbb{P}\Bigg( \bigg\{\Delta_a - 2c_{a,t} < 0 \bigg\}   \Bigg) = 0~.
\end{align}
Hence, 
$$ E[N'_{a,T}]  \leq \sum_{t = L}^{T-1} \frac{4}{t^2} \leq \sum_{t = 1}^\infty \frac{4}{t^2} \leq \frac{2\pi^2}{3}~. $$
\end{proof}
Now we bound the expected cumulative regret of \cucbtwo. From Equation \ref{eq:cumulative regret without cost} in Section \ref{sec: model and notation}, we have at the end of $T$ rounds
\begin{align*}
   E[R_{\CUCBTwo}(T)] &= T\mu_{a^*} -  \sum_{a\in \mathcal{A}} \mu_a E[N_{a,T}] \\
                    &= \sum_{a\in \mathcal{A}} \Delta_a E[N_{a,T}] \leq  \sum_{a\in \mathcal{A}} \Delta_a(L_a + \frac{2\pi^2}{3})~.
\end{align*}
The inequality in the last line of the above equation follows from Lemma \ref{lemma: the bound on the number of sub-optimal interventions}.

\section*{Acknowledgements}
Vineet Nair is thankful to be supported by the European Union’s Horizon 2020 research and innovation program under grant agreement No 682203 -ERC-[ Inf-Speed-Tradeoff]. Vishakha Patil gratefully acknowledges the support of a Google PhD Fellowship.

\bibliographystyle{alpha}
\bibliography{references}

\appendix

\end{document}